\documentclass{tlp}
\usepackage[english]{babel}
\usepackage{amsmath,amssymb}

\submitted{20 March 2012}
\revised{19 August 2013}
\accepted{5 November 2013}

\newcommand{\resp}{\emph{resp.}\ }
\newcommand{\condPi}[2]{\cond{\Pi}{#1}{#2}}
\newcommand{\condN}[2]{\cond{N}{#1}{#2}}
\newcommand{\cond}[3]{\ensuremath{#1\left(#2\mid#3\right)}}
\newcommand{\complexity}[1]{\ensuremath{\mathsf{#1}}}
\newcommand{\cP}[0] {\complexity{P}}
\newcommand{\cC}[0] {\complexity{C}}
\newcommand{\cNP}[0]{\complexity{NP}}
\newcommand{\cco}[0]{\complexity{co}}
\newcommand{\ccoNP}[0]{\cco\cNP}
\newcommand{\hisig}[2]{\ensuremath{\Sigma_{#1}^{#2}}}
\newcommand{\hipi}[2]{\ensuremath{\Pi_{#1}^{#2}}}

\newcommand{\spolpi}[1]{\hipi{#1}{\mathbf{\cP}}}

\newcommand{\spolsig}[1]{\hisig{#1}{\mathbf{\cP}}}
\newcommand{\tlabel}[1]{\ensuremath{{}^{(#1)}}}
\newcommand{\tref}[2]{\tlabel{#1} #2}
\newcommand{\aset}[1] {\ensuremath{\left\{#1\right\}}}
\providecommand{\set}[1] {\aset{#1}}
\newcommand{\condset}[2]{\aset{{#1}\mid{#2}}}
\newcommand{\range}[2] {\rangec{#1}{#2}{,}}
\newcommand{\rangec}[3] {\ensuremath{#1#3...#3#2}}
\newcommand{\prange}[3] {\range{#1 #2}{#1 #3}}

\newcommand{\srange}[3] {\range{#1_{#2}}{#1_{#3}}}
\newcommand{\psrange}[4] {\prange{#1}{{#2}_{#3}}{{#2}_{#4}}}
\newcommand{\srangec}[4] {\rangec{#1_{#2}}{#1_{#3}}{#4}}
\newcommand{\psrangec}[5] {\rangec{#1 #2_{#3}}{#1 #2_{#4}}{#5}}
\providecommand{\suchthat}{\ensuremath{ \cdot }}
\newcommand{\sForall}[2]{\ensuremath{\forall #1 \suchthat #2}}

\newcommand{\function}[3]{\ensuremath{#1 : #2\lothen#3}}
\newcommand{\eg}[0] {e.g.~}
\newcommand{\ie}[0] {i.e.~}
\newcommand{\wrt}[0] {w.r.t.~}
\newcommand{\concept}[1]{\emph{#1}}

\newcommand{\pointtosec}[1]{Section~\ref{sec:#1}}
\newcommand{\pointtoex}[1]{Example~\ref{ex:#1}}
\newcommand{\pointtodef}[1]{Definition~\ref{def:#1}}

\newcommand{\pointtoprop}[1]{Proposition~\ref{prop:#1}}
\newcommand{\pointtocor}[1]{Corollary~\ref{cor:#1}}
\newcommand{\pointtotbl}[1]{Table~\ref{tbl:#1}}

\newtheorem{corollary}{Corollary}
\newtheorem{proposition}{Proposition}
\newtheorem{lemma}{Lemma}
\newtheorem{definition}{Definition}
\newtheorem{example}{Example}
\newcommand{\naf}[0]{not~}
\newcommand{\snaf}[0]{not}
\newcommand{\arule}[2]{\ensuremath{#1 \leftarrow #2}}

\newcommand{\hbase}[1]{\ensuremath{\mathcal{B}_{#1}}}
\newcommand{\hlit}[1]{\ensuremath{Lit_{#1}}}
\newcommand{\hclause}[1]{\ensuremath{\mathit{Clause}_{#1}}}
\newcommand{\srule}[1]{\arule{#1_0}{\sbody{#1}{1}{m}}}
\newcommand{\grule}[1]{\arule{#1_0}{\gbody{#1}{1}{m}{n}}}
\newcommand{\sbodyi}[3]{\range{#1_{#2}}{#1_{#3}}}
\newcommand{\gbodyi}[4]{\sbodyi{#1}{#2}{#3}, \prange{\naf}{#1_{#3+1}}{#1_{#4}}}
\newcommand{\sbody}[3]{\sbodyi{#1}{#2}{#3}}
\newcommand{\gbody}[4]{\gbodyi{#1}{#2}{#3}{#4}}

\newcommand{\drule}[1]{\arule{\rangec{#1_0}{#1_k}{;}}{\gbody{#1}{k+1}{m}{n}}}

\newcommand{\prule}[3]{\ensuremath{\textnormal{\textbf{#1: }}\arule{\mathit{#2}}{\mathit{#3}}}}
\newcommand{\parule}[3]{\ensuremath{\textnormal{\textbf{#1: }}&\arule{\mathit{#2}}{\mathit{#3}}}}
\newcommand{\pdrule}[1]{\ensuremath{\mathbf{\boldsymbol\lambda:\ }\drule{#1}}}
\newcommand{\plit}[2]{{#1}^{\mathit{#2}}}
\newcommand{\loand}[0]{\ensuremath{\wedge}}
\newcommand{\lothen}[0]{\ensuremath{\rightarrow}}

\newcommand{\prulealign}[3]{\ensuremath{\mathbf{#1\!:\!\ }&\arule{#2}{#3}}}

\newcommand{\proofat}[1]{The proof is given in the online appendix of the paper, pp. #1.}

 \title[Characterizing and Extending ASP using Possibility Theory]{Characterizing and Extending Answer Set Semantics using Possibility Theory}
 
  \author[Kim Bauters et al.]
         {KIM BAUTERS\\
         Department of Applied Mathematics, Computer Science and Statistics, Universiteit Gent\\ Krijgslaan 281 (WE02), 9000 Gent, Belgium\\
         \email{kim.bauters@gmail.com}
         \and
         STEVEN SCHOCKAERT\\
         School of Computer Science \& Informatics, Cardiff University\\ 5 The Parade, Cardiff CF24 3AA, United Kingdom\\
         \email{s.schockaert@cs.cardiff.ac.uk}
         \and
         MARTINE DE COCK\\
         Department of Applied Mathematics, Computer Science and Statistics, Universiteit Gent\\ Krijgslaan 281 (WE02), 9000 Gent, Belgium\\
         \email{martine.decock@ugent.be}
         \and
         DIRK VERMEIR\\
         Department of Computer Science, Vrije Universiteit Brussel,\\ Pleinlaan 2, 1050 Brussel, Belgium\\
         \email{dirk.vermeir@vub.ac.be}}

\begin{document}
\maketitle

\begin{abstract}
Answer Set Programming (ASP) is a popular framework for modeling combinatorial problems. However, ASP cannot easily be used for reasoning about uncertain information. Possibilistic ASP (PASP) is an extension of ASP that combines possibilistic logic and ASP. In PASP a weight is associated with each rule, where this weight is interpreted as the certainty with which the conclusion can be established when the body is known to hold. As such, it allows us to model and reason about uncertain information in an intuitive way. In this paper we present new semantics for PASP, in which rules are interpreted as constraints on possibility distributions. 
Special models of these constraints are then identified as possibilistic answer sets. In addition, since ASP is a special case of PASP in which all the rules are entirely certain, we obtain a new characterization of ASP in terms of constraints on possibility distributions.
This allows us to uncover a new form of disjunction, called weak disjunction, that has not been previously considered in the literature. In addition to introducing and motivating the semantics of weak disjunction, we also pinpoint its computational complexity. In~particular, while the complexity of most reasoning tasks coincides with standard disjunctive ASP, we find that brave reasoning for programs with weak disjunctions is easier.
\end{abstract}



\section{Introduction}\label{sec:introduction}
Answer set programming (ASP) is a form of logic programming with a fully declarative semantics, centered around the notion of a stable model. Syntactically, an ASP program is a set of rules of the form $(\arule{head}{body})$ where $head$ is true whenever $body$ is true. Possibilistic ASP (PASP) extends upon ASP by associating a weight with every rule, which is interpreted as the necessity with which we can derive the head of the rule when the body is known to hold.
Semantics for PASP have been introduced in~\cite{nicolas:possibilistic} for possibilistic normal programs and later extended to possibilistic disjunctive programs in~\cite{nieves:semantics}. Under these semantics, a~possibilistic rule with certainty~$\lambda$ allows us to derive $head$ with certainty $\min(\lambda, N(body))$ where $N(body)$ denotes the necessity of the body, \ie the certainty of $head$ is restricted by the least certain piece of information in the derivation chain. 
Specifically, to deal with PASP rules without negation-as-failure, the semantics from~\cite{nicolas:possibilistic} treat such rules as implications in possibilistic logic~\cite{dubois:possibilistic}. 
When faced with negation-as-failure, the semantics from~\cite{nicolas:possibilistic} rely on the reduct operation from classical ASP. Essentially, this means that the weights associated with the rules are initially ignored, the classical reduct is determined and the weights are then reassociated with the corresponding rules in the reduct.
Given this particular treatment of negation-as-failure, the underlying intuition of `$\naf l$' is ``$\emph{`$l$' cannot be derived with a strictly positive certainty}$''. Indeed, as soon as `$l$' can be derived with a certainty $\lambda > 0$, `$l$' is treated as true when determining the reduct. However, this particular understanding of negation-as-failure is not always the most intuitive one.

Consider the following example. You want to go to the airport, but you notice that your passport will expire in less than three months. Some countries require that the passport is at least valid for an additional three months on the date of entry. As~such, you have some certainty that your passport might be invalid ($\mathit{invalid}$). 
When you are not entirely certain that your passport is invalid, you should still go to the airport ($\mathit{airport}$) and check-in nonetheless. Indeed, since you are not absolutely certain that you will not be allowed to board, you might still get lucky. We have the possibilistic program:
  \begin{align*}
    \parule{0.1}{\mathit{invalid}}{}\\
    \parule{1}{\mathit{airport}}{\naf \mathit{invalid}}
  \end{align*}
where $0.1$ and $1$ are the weights associated with the rules (\arule{\mathit{invalid}}{}) and $\arule{\mathit{airport}}{\mathit{invalid}}$, respectively. Clearly, what we would like to be able to conclude with a high certainty is that you need to go to the airport to check-in. However, as the semantics from~\cite{nicolas:possibilistic} adhere to a different intuition of negation-as-failure, the conclusion is that you need to go to the airport with a necessity of $0$. Or, in other words, you should not go to the airport at all. 

As a first contribution in this paper, we present new semantics for PASP by interpreting possibilistic rules as constraints on possibility distributions. These semantics do not correspond with the semantics from \cite{nicolas:possibilistic} when considering programs with negation-as-failure. Specifically, the semantics presented in this paper can be used in settings in which the possibilistic answer sets according to~\cite{nicolas:possibilistic} do not correspond with the intuitively acceptable results. For the example mentioned above, the conclusion under the new semantics is that you need to go to the airport with a necessity of $0.9$. 

In addition, the new semantics allow us to uncover a new characterization of ASP in terms of possibility theory. Over the years, many equivalent approaches have been proposed to define the notion of an answer set. One of the most popular characterizations is in terms of the Gelfond-Lifschitz reduct~\cite{gelfond:stablemodel} in which an answer set is guessed and verified to be stable. This characterization is used in the semantics for PASP as presented in \cite{nicolas:possibilistic}. Alternatively, the answer set semantics of normal programs can be defined in terms of autoepistemic logic \cite{marek:autoepistemic}, a well-known non-monotonic modal logic.  An~important advantage of the latter approach is that autoepistemic logic enjoys more syntactic freedom, which opens the door to more expressive forms of logic programming.  However, as has been shown early on in~\cite{lifschitz:extended}, the characterization in terms of autoepistemic logic does not allow us to treat classical negation or disjunctive rules in a natural way, which weakens its position as a candidate for generalizing ASP from normal programs to \eg disjunctive programs.  Equilibrium logic~\cite{pearce:equilibrium} offers yet another way for characterizing and extending ASP, but does not feature modalities which limits its potential for epistemic reasoning as it does not allow us to reason over the established knowledge of an agent. The new characterization of ASP, as presented in this paper, is a characterization in terms of necessary and contingent truths, where possibility theory is used to express our certainty in logical propositions. Such a characterization is unearthed by looking at ASP as a special case of PASP in which the rules are certain and no uncertainty is allowed in the answer sets. 
It highlights the intuition of ASP that the head of a rule is certain when the information encoded in its body is certain. Furthermore, this characterization stays close to the intuition of the Gelfond-Lifschitz reduct, while sharing the explicit reference to modalities with autoepistemic logic. 

As a second contribution, we show in this paper how this new characterization of ASP in terms of possibility theory can be used to uncover a new form of disjunction in both ASP and PASP. As indicated, we have that the new semantics offer us an explicit reference to modalities, \ie operators with which we can qualify a statement. Epistemic logic is an example of a modal logic in which we use the modal operator~$K$ to reason about knowledge, where $K$ is intuitively understood as ``\emph{we know that}''. A~statement such as $a \lor b \lor c$ can then be treated in two distinct ways. On the one hand, we can interpret this statement as $Ka \lor Kb \lor Kc$, which makes it explicit that we know that one of the disjuncts is true. This treatment corresponds with the understanding of disjunction in disjunctive ASP and will be referred to as \concept{strong disjunction}. Alternatively, we can interpret $a \lor b \lor c$ as $K(a \lor b \lor c)$ which only states that we know that the disjunction is true, \ie we do not know which of the disjuncts is true. We will refer to this form of disjunction as \concept{weak disjunction}. This is the new form of disjunction that we will discuss in this paper, as it allows us to reason in settings where a choice cannot or should not be made. Still, such a framework allows for non-trivial forms of reasoning. 

Consider the following example. 
A~SCADA (supervisory control and data acquisition) system is used to monitor the brewing of beer in an industrialised setting. To control the fermentation, the system regularly verifies an air-lock for the presence of bubbles. An absence of bubbles may be due to a number of possible causes. On the one hand there may be a production problem such as a low yeast count or low temperature. Adding yeast when the temperature is low results in a beer with a strong yeast flavour, which should be avoided. Raising the temperature when there is too little yeast present will kill off the remaining yeast and will ruin the entire batch. On the other hand, there may be technical problems. There may be a malfunction in the SCADA system, which can be verified by running a diagnostic. The operator runs a diagnostic ($\mathit{diagnostic}$), which reports back that there is no malfunction ($\neg\mathit{malfunction}$). Or, alternatively, the air-lock may not be sealed correctly ($\mathit{noseal}$). The operator furthermore checks the temperature because he suspects that the temperature is the problem ($\mathit{verifytemp}$), but the defective temperature sensor returns no temperature when checked ($\mathit{notemp}$). These three technical problems require physical maintenance and the operator should send someone out to fix them. Technical problems do not affect the brewing. As such, the brewing process should not be interrupted for such problems as this will ruin the current batch. If there is a production problem, however, the brewing process needs to be interrupted as soon as possible (in addition, evidently, to interrupting the brewing  process when the brewing is done). This prevents the current batch from being ruined due to over-brewing but also allows the interaction with the contents of the kettle. In particular, when the problem is diagnosed to be low yeast the solution is to add a new batch of yeast and restart the process. Similarly, low temperature can be solved by raising the kettle temperature and restarting the fermentation process. Obviously, the goal is to avoid ruining the current batch. An employer radios in that the seal is okay. We have the following program:
{
\allowdisplaybreaks
\begin{align*}
  \arule{\mathit{lowyeast} \lor \mathit{lowtemp} \lor \mathit{noseal} \lor \mathit{malfunction}&}{\naf bubbles}\\
  \arule{\mathit{diagnostic}&}{}\\
  \arule{\neg \mathit{malfunction}&}{\mathit{diagnostic}}\\
  \arule{\mathit{verifytemp}&}{}\\
  \arule{\mathit{\mathit{notemp}}&}{\mathit{verifytemp}}\\
  \arule{\mathit{maintenance}&}{\mathit{noseal} \lor \mathit{malfunction} \lor \mathit{notemp}}\\
  \arule{\mathit{brew}&}{\naf (\mathit{lowyeast} \lor \mathit{lowtemp} \lor \mathit{done})}\\
  \arule{\mathit{addyeast}&}{\mathit{lowyeast}}\\
  \arule{\mathit{raisetemp}&}{\mathit{lowtemp}}\\
  \arule{\mathit{ruin}&}{\mathit{raisetemp}, \naf \mathit{lowtemp}}\\
  \arule{\mathit{ruin}&}{\mathit{addyeast}, \naf \mathit{lowyeast}}\\
  \arule{\mathit{ruin}&}{\naf \mathit{brew}, \naf (\mathit{lowtemp} \lor \mathit{lowyeast})}\\
  \arule{&}{\mathit{ruin}}\\
  \arule{\neg \mathit{noseal}&}{}
\end{align*}
}
The program above does not use the standard ASP syntax since we allow for disjunction in the body. Furthermore, the disjunction used in the head and the body is weak disjunction. The only information that we can therefore deduce from \eg the first rule is ($\mathit{lowyeast} \lor \mathit{lowtemp} \lor \mathit{noseal} \lor \mathit{malfunction}$). 
At first, this new form of disjunction may indeed appear weaker that strong disjunction since it does not induce a choice. Still, even without inducing a choice, conclusions obtained from other rules may allow us to refine our knowledge. In particular, note that from $\mathit{lowyeast} \lor \mathit{lowtemp} \lor \mathit{noseal} \lor \mathit{malfunction}$ together with $\neg \mathit{malfunction}$ and $\neg \mathit{noseal}$ we can entail $\mathit{lowyeast} \lor \mathit{lowtemp}$. Similarly, conclusions can also have prerequisites that are disjunctions. For example, we can no longer deduce $\mathit{brew}$ since $\mathit{lowyeast} \lor \mathit{lowtemp}$ entails $\mathit{lowyeast} \lor \mathit{lowtemp} \lor \mathit{done}$. From $\arule{\mathit{maintenance}}{\mathit{noseal} \lor \mathit{malfunction} \lor \mathit{notemp}}$ and $\mathit{notemp}$ we can deduce that we should call maintenance. However, we do not yet have enough information to diagnose whether yeast should be added or whether the temperature should be raised.
The unique answer set of this program, according to the semantics of weak disjunction which we present in \pointtosec{disjunction}, is given by
\begin{align*}
  \{&\mathit{lowyeast} \lor \mathit{lowtemp}, \mathit{maintenance}, \\
  &\mathit{diagnostic}, \neg \mathit{malfunction}, \mathit{verifytemp}, \mathit{notemp}, \neg \mathit{noseal}\}
\end{align*}

The expressiveness of weak disjunction becomes clear when we study its complexity. In~particular, we show that while most complexity results coincide with the strong disjunctive semantics, the complexity of brave reasoning (deciding whether a literal `$l$' is entailed by a consistent answer set of program~$P$) in absence of negation-as-failure is lower for weak disjunction. Still, the expressiveness is higher than for normal programs. The complexity results are summarized in \pointtotbl{results-extra} in \pointtosec{complexity}.

The remainder of this paper is organized as follows. In \pointtosec{background} we provide the reader with some important notions from answer set programming and possibilistic logic. In \pointtosec{characterization} we introduce new semantics for PASP which can furthermore be used to characterize normal ASP programs using possibility theory. In \pointtosec{disjunction} we characterize disjunctive ASP in terms of constraints on possibility distributions and we discuss the complexity results of the new semantics for PASP in detail in \pointtosec{complexity}. Related work is discussed in \pointtosec{related} and we formulate our conclusions in \pointtosec{conclusions}.

This paper aggregates and extends parts of our work from~\cite{bauters:weak} and substantially extends a previous conference paper~\cite{bauters:possibilistic}, which did not consider classical negation nor computational complexity. In addition, rather than limiting ourselves to atoms in this paper, we extend our work to cover the case of literals, which offer interesting and unexpected results in the face of weak disjunction. Complexity results are added for all reasoning tasks and full proofs are provided in appendix.


\section{Background}\label{sec:background}
We start by reviewing the definitions from both answer set programming and possibilistic logic that will be used in the remainder of the paper. We then review the semantics of PASP from~\cite{nicolas:possibilistic}, a framework that combines possibilistic logic and ASP. Finally, we recall some notions from complexity theory.

\subsection{Answer Set Programming}\label{sec:background:asp}
To define ASP programs, we start from a finite set of atoms $\mathcal{A}$. A \concept{literal} is defined as an atom $a$ or its classical negation $\neg a$. For $L$ a set of literals, we use $\neg L$ to denote the set $\condset{\neg l}{l \in L}$ where, by definition, $\neg \neg a = a$. A set of literals $L$ is \concept{consistent} if $L \cap \neg L = \emptyset$. We write the set of all literals as $\mathcal{L} = (\mathcal{A} \cup \neg\mathcal{A})$. A~\concept{naf-literal} is either a literal `$l$' or a literal `$l$' preceded by $\snaf$, which we call the \concept{negation-as-failure operator}. Intuitively, `$\naf l$' is true when we cannot prove `$l$'. An expression of the form 
$$
\drule{l}
$$
with $l_i$ a literal for every $0 \leq i \leq n$, is called a \concept{disjunctive rule}. We call $\srangec{l}{0}{k}{;}$ the \concept{head} of the rule (interpreted as a disjunction) and $\gbody{l}{k+1}{m}{n}$ the \concept{body} of the rule (interpreted as a conjunction). For a rule $r$ we use $head(r)$ and $body(r)$ to denote the set of literals in the head, \resp the body. Specifically, we use $body_+(r)$ to denote the set of literals in the body that are not preceded by the negation-as-failure operator `$\snaf$' and $body_-(r)$ for those literals that are preceded by `$\snaf$'. Whenever a disjunctive rule does not contain negation-as-failure, \ie when $n=m$, we say that it is a \concept{positive} disjunctive rule. A rule with an empty body, \ie a rule of the form ($\arule{\srangec{l}{0}{k}{;}}{}$), is called a \concept{fact} and is used as a shorthand for ($\arule{\srangec{l}{0}{k}{;}}{\top}$) with $\top$ a special language construct that denotes tautology. A rule with an empty head, \ie a rule of the form ${(\arule{}{\srange{l}{k+1}{m}}, \psrange{\naf}{l}{m+1}{n})}$, is called a \concept{constraint rule} and is used as a shorthand for the rule of the form (${\arule{\bot}{\srange{l}{k+1}{m}}, \psrange{\naf}{l}{m+1}{n}}$) with $\bot$ a special language construct that denotes contradiction. 

A~(positive) disjunctive program $P$ is a set of (positive) disjunctive rules. A~\concept{normal rule} is a disjunctive rule with at most one literal in the head. A~\concept{simple rule} is a normal rule with no negation-as-failure. A~\concept{definite rule} is a simple rule with no classical negation, \ie  in which all literals are atoms. A \concept{normal (\resp simple, definite) program} $P$ is a set of normal (\resp simple, definite) rules.

The  \concept{Herbrand base} \hbase{P} of a disjunctive program $P$ is the set of atoms appearing in~$P$. We define the set of literals that are relevant for a disjunctive program $P$ as $\hlit{P}= (\hbase{P} \cup \neg\hbase{P})$. An~\concept{interpretation} $I$ of a disjunctive program $P$ is any set of literals $I \subseteq \hlit{P}$. A \concept{consistent interpretation} $I$ is an interpretation $I$ that does not contain both $a$ and $\neg a$ for some $a \in I$.

A \emph{consistent interpretation} $I$ is said to be a \concept{model} of a positive disjunctive rule $r$ if $head(r) \cap I \neq \emptyset$ or $body(r) \not\subseteq I$, \ie the body is false or the head is true. In particular, a consistent interpretation $I$ is a \concept{model} of a constraint rule $r$ if $body(r) \not \subseteq I$. If for an interpretation $I$ and a constraint rule $r$ we have that $body(r) \subseteq I$, then we say that the interpretation $I$ \concept{violates} the constraint rule $r$. Notice that for a fact rule we require that $head(r) \cap I \neq \emptyset$, \ie at least one of the literals in the head must be true. Indeed, otherwise $I$ would not be a model of $r$. An~interpretation $I$ of a positive disjunctive program $P$ is a model of $P$ either if $I$ is consistent and for every rule $r\in P$ we have that $I$ is a model of $r$, or if $I = \hlit{P}$. It follows from this definition that $\hlit{P}$ is always a model of $P$, and that all other models of $P$ (if any) are consistent interpretations, which we will further on also refer to as \concept{consistent models}. We say that $I$ is an \emph{answer set} of the positive disjunctive program $P$ when~$I$ is a minimal model of $P$ \wrt set inclusion.

The semantics of an ASP program with negation-as-failure is based on the idea of a stable model~\cite{gelfond:stablemodel}. The reduct $P^I$ of a disjunctive program $P$ \wrt the interpretation $I$ is defined as:
\begin{align*}
  P^I  = &\{\arule{l_0; \ldots; l_k}{\srange{l}{k+1}{m}}~\mid~(\set{\srange{l}{m+1}{n}} \cap I = \emptyset)\\
  &~\wedge (\drule{l})\in P\}.
\end{align*}
An interpretation $I$ is said to be an answer set of the disjunctive program $P$ when $I$ is an answer set of the positive disjunctive program $P^I$ (hence the notion of stable model). Note that we can also write the disjunctive program $P$ as $P = P' \cup C$ where $C$ is the set of constraint rules in $P$. An interpretation $I$ then is an answer set of the disjunctive program $P$ when $I$ is an answer set of $P'$ and $I$ is a model of $C$, \ie $I$ does not violate any constraints in $C$.
Whenever $P$ has \concept{consistent answer sets}, \ie answer sets that are consistent interpretations, we say that $P$ is a \concept{consistent program}. When $P$ has the answer set $\hlit{P}$, then this is the unique~\cite{baral:knowledge} inconsistent answer set and we say that $P$ is an \concept{inconsistent program}.

Answer sets of simple programs can also be defined in a more procedural way. By using the \concept{immediate consequence operator} $T_P$, which is defined for a simple program $P$ without constraint rules and \wrt an interpretation $I$ as:
  \begin{align*}
    T_P(I) = \condset{l_0}{(\srule{l}) \in P \loand \set{\srange{l}{1}{m}} \subseteq I }.
  \end{align*}
  
We use $P^{\star}$ to denote the fixpoint which is obtained by repeatedly applying $T_P$ starting from the empty interpretation $\emptyset$, \ie it is the least fixpoint of $T_P$ \wrt set inclusion. When the interpretation $P^{\star}$ is consistent, $P^{\star}$ is the (unique and consistent) answer set of the simple program P without constraint rules. When we allow constraint rules, an interpretation is a (consistent) answer set of $P = P' \cup C$ iff $I$ is a (consistent) answer set of $P$ and $I$ is a model of $C$. For both simple and normal programs, with or without constraint rules, we have that $\hlit{P}$ is the (unique and inconsistent) answer set of $P$ if $P$ has no consistent answer set(s).

\subsection{Possibilistic Logic}\label{sec:background:pl}
An interpretation in possibilistic logic corresponds with the notion of an interpretation in propositional logic. We represent such an interpretation as a set of atoms $\omega$, where $\omega \models a$ if $a \in \omega$ and $\omega \models \neg a$ otherwise with $\models$ the satisfaction relation from classical logic.
The set of all interpretations is defined as $\Omega = 2^{\mathcal{A}}$, with $\mathcal{A}$ a finite set of atoms.
At the semantic level, possibilistic logic~\cite{dubois:possibilistic} is defined in terms of a \concept{possibility distribution} $\pi$ on the universe of interpretations. A possibility distribution, which is an $\Omega \lothen [0,1]$ mapping, encodes for each interpretation (or world) $\omega$ to what extent it is plausible that $\omega$ is the actual world. By~convention, $\pi(\omega) = 0$ means that~$\omega$ is impossible and $\pi(\omega) = 1$ means that no available information prevents~$\omega$ from being the actual world. A possibility distribution $\pi$ is said to be \concept{normalized} if $\exists \omega \in \Omega \cdot \pi(\omega) = 1$, \ie at least one interpretation is entirely plausible. We say that a possibility distribution $\pi$ is \concept{vacuous} when $\forall \omega \in \Omega \cdot \pi(\omega) = 0$. Note that possibility degrees are mainly interpreted qualitatively:  when $\pi(\omega) > \pi(\omega')$, $\omega$~is~considered more plausible than~$\omega'$. For two possibility distributions $\pi_1$ and $\pi_2$ with the same domain $\Omega$ we write $\pi_1 \geq \pi_2$ when $\forall \omega \in \Omega \cdot \pi_1(\omega) \geq \pi_2(\omega)$ and we write $\pi_1 > \pi_2$ when $\pi_1 \geq \pi_2$ and $\pi_1 \neq \pi_2$. 

A possibility distribution $\pi$ induces two uncertainty measures that allow us to rank propositions. The \concept{possibility measure} $\Pi$ is defined by~\cite{dubois:possibilistic}:
\begin{align*}
  \Pi(p) &= \max\condset{\pi(\omega)}{\omega \models p}
\end{align*}
and evaluates the extent to which a proposition $p$ is consistent with the beliefs expressed by $\pi$. The dual \emph{necessity measure} $N$ is defined by:
\begin{align*}
  N(p) &= 1 - \Pi(\neg p)
\end{align*}
and evaluates the extent to which a proposition $p$ is entailed by the available beliefs~\cite{dubois:possibilistic}. Note that we always have $N(\top) = 1$ for any possibility distribution, while $\Pi(\top) = 1$ (and, related, $N(\bot) = 0$) only holds when the possibility distribution is normalized (\ie only normalized possibility distributions can express consistent beliefs)~\cite{dubois:possibilistic}. To identify the possibility/necessity measure associated with a specific possibility distribution $\pi_\mathrm{X}$, we will use a subscript notation, \ie $\Pi_\mathrm{X}$ and $N_\mathrm{X}$ are the corresponding possibility and necessity measure, respectively. We omit the subscript when the possibility distribution is clear from the context.

 An important property of necessity measures is the min-decomposability property \wrt conjunction: $N(p\wedge q) = \min (N(p),N(q))$ for all propositions $p$ and $q$. However, for disjunction only the inequality $N(p \vee q) \geq \max (N(p), N(q))$ holds. As possibility measures are the dual measures of necessity measures, they have the property of max-decomposability \wrt disjunction, whereas for the conjunction only the inequality $\Pi(p \wedge q) \leq \min \left(\Pi(p), \Pi(q)\right)$ holds.

At the syntactic level, a \concept{possibilistic knowledge base} consists of pairs $(p,c)$ where $p$ is a propositional formula and $c\in\ ]0,1]$ expresses the certainty that $p$ is the case. Formulas of the form $(p,0)$ are not explicitly represented in the knowledge base since they encode trivial information. A~formula $(p,c)$ is interpreted as the constraint $N(p)\geq c$, \ie a possibilistic knowledge base $\Sigma$ corresponds to a set of constraints on possibility distributions. 
Typically, there can be many possibility distributions that satisfy these constraints. In practice, we are usually only interested in the \concept{least specific possibility distribution}, which is the possibility distribution that makes minimal commitments, \ie the greatest possibility distribution \wrt the ordering~$>$ defined above. Such a least specific possibility distribution always exists and is unique~\cite{dubois:possibilistic}.  

In \pointtosec{disjunction} we will also consider constraints that deviate from the form of constraints we just discussed. As a result, there can be multiple minimally specific possibility distributions rather than a unique least specific possibility distribution. To increase the uniformity throughout the paper we immediately start using the concept of a \concept{minimally specific possibility distribution}, which is a maximal possibility distribution \wrt the ordering $>$, even though the distinction between the least specific possibility distribution and minimally specific possibility distributions only becomes relevant once we discuss the characterization of disjunctive programs.

\subsection{Possibilistic Answer Set Programming}
Possibilistic ASP (PASP)~\cite{nicolas:possibilistic} combines ASP and possibility theory by associating a weight with each rule, where the weight denotes the necessity with which the head of the rule can be concluded given that the body is known to hold. If it is uncertain whether the body holds, the necessity with which the head can be derived is the minimum of the weight associated with the rule and the degree to which the body is necessarily true. 

Syntactically, a possibilistic disjunctive (\resp normal, simple, definite) program is a set of pairs $p = (r,\lambda)$ with~$r$ a disjunctive (\resp normal, simple, definite) rule and $\lambda \in\ ]0,1]$ a certainty associated with $r$. Possibilistic rules with $\lambda = 0$ are generally omitted as only trivial information can be derived from them.
We will also write a possibilistic rule $p = (r,\lambda)$ with $r$ a disjunctive rule of the form $(\drule{l})$ as: \[ \pdrule{l}. \]
For a possibilistic rule $p = (r,\lambda)$ we use $p^*$ to denote $r$, \ie the classical rule obtained by ignoring the certainty. Similarly, for a possibilistic program $P$ we use $P^*$ to denote the set of rules $\condset{p^*}{p \in P}$. The set of all weights found in a possibilistic program $P$ is denoted by $cert(P) = \condset{\lambda}{p = (r,\lambda) \in P}$. We will also use the extended set of weights $cert^+(P)$, defined as
$cert^+(P) = \condset{\lambda}{\lambda \in cert(P)} \cup \condset{1-\lambda}{\lambda \in cert(P)} \cup \set{0, \frac{1}{2}, 1}$.

Semantically, PASP is based on a generalization of the concept of an interpretation. In~classical ASP, an interpretation can be seen as a mapping $\function{I}{\hlit{P}}{\set{0,1}}$, \ie a literal $l \in \hlit{P}$ is either true or false. This notion is generalized in PASP to a \concept{valuation}, which is a function ${\function{V}{\hlit{P}}{[0,1]}}$. The underlying intuition of $V(l) = \lambda$ is that the literal `$l$' is true with certainty `$\lambda$', which we will also write in set notation as $\plit{l}{\lambda} \in V$. As such, a valuation corresponds with the set of constraints $\condset{N(l) \geq \lambda}{\plit{l}{\lambda} \in V}$.
 Note that, like interpretations in ASP, these valuations are of an epistemic nature, \ie they reflect what we know about the truth of atoms. For notational convenience, we often also use the set notation $V = \set{\plit{l}{\lambda}, \ldots}$. In~accordance with this set notation, we write $V = \emptyset$ to denote the valuation in which each literal is mapped to $0$. For $\lambda \in [0,1]$ a certainty and $V$ a valuation, we use $V^\lambda$ to denote the classical projection $\condset{l}{l \in \hlit{P}, V(l) \geq \lambda}$. We also use $V^{\underline{\lambda}} = \condset{l}{l \in \hlit{P}, V(l) > \lambda}$, \ie those literals that can be derived to be true with certainty strictly greater than `$\lambda$'. A valuation is said to be \emph{consistent} when $V^{\underline{0}}$ is consistent. In such a case, there always exists a normalized possibility distribution $\pi_V$ such that $N_V(l) = V(l)$.

We now present a straightforward extension of the semantics for PASP introduced in~\cite{nicolas:possibilistic}. Let the $\lambda$-cut $P_\lambda$ of a possibilistic program $P$, with $\lambda \in [0,1]$, be defined as: \[ P_\lambda = \condset{r}{(r,\lambda') \in P \text{ and } \lambda' \geq \lambda}, \]
\ie the rules in $P$ with an associated certainty higher than or equal to `$\lambda$'.
\begin{definition}\label{def:nicolas:immediate}
  Let $P$ be a possibilistic simple program and $V$ a valuation. The immediate consequence operator $T_P$ is defined as: \[ T_P(V)(l_0) = \max\condset{\lambda \in [0,1]}{V^\lambda \models \srange{l}{1}{m} \text{ and } (\srule{l}) \in P_\lambda}. \]
\end{definition}

\noindent The intuition of \pointtodef{nicolas:immediate} is that we can derive the head only with the certainty of the weakest piece of information, \ie the necessity of the conclusion is restricted either by the certainty of the rule itself or the lowest certainty of the literals used in the body of the rule. Note that the immediate consequence operator defined in \pointtodef{nicolas:immediate} is equivalent to the one proposed in~\cite{nicolas:possibilistic}, although we formulate it somewhat differently. Also, the work from~\cite{nicolas:possibilistic} only considered definite programs, even though adding classical negation does not impose any problems.

As before, we use $P^\star$ to denote the fixpoint obtained by repeatedly applying $T_P$ starting from the minimal valuation $V = \emptyset$, \ie the least fixpoint of $T_P$ \wrt set inclusion. A valuation $V$ is said to be the answer set of a possibilistic simple program if $V = P^\star$ and $V$ is consistent. Answer sets of possibilistic normal programs are defined using a reduct. Let $L$ be a set of literals. The reduct $P^L$ of a possibilistic normal program is defined as~\cite{nicolas:possibilistic}: \[ P^L = \condset{ (\arule{head(r)}{body_+(r)}, \lambda)}{ (r,\lambda) \in P \text{ and } body_-(r) \cap L = \emptyset}. \] A consistent valuation $V$ is said to be a possibilistic answer set of the possibilistic normal program $P$ iff $\left( P^{(V^{\underline{0}})} \right)^{\star} = V$, \ie if $V$ is the answer set of the reduct $P^{(V^{\underline{0}})}$. 

\begin{example}\label{ex:abnormal}Consider the possibilistic normal program $P$ from the introduction:
  \begin{align*}
    \parule{0.1}{\mathit{invalid}}{}\\
    \parule{1}{\mathit{airport}}{\naf \mathit{invalid}}
  \end{align*}
  It is easy to verify that \set{\plit{\mathit{invalid}}{0.1}} is a possibilistic answer set of $P$. Indeed, $P^{\set{\mathit{invalid}}}$ is the set of rules:
  \begin{align*}
    \parule{0.1}{\mathit{invalid}}{}
  \end{align*}
  from which it trivially follows that ${(P^{\set{\mathit{invalid}}})}^\star = \set{\plit{\mathit{invalid}}{0.1}}$. The conclusion is thus that we do not need to go to the airport, which differs from our intuition of the problem. We will revisit this example in \pointtoex{airport:again} in \pointtosec{characterizing:normal}.
\end{example}

The semantics we presented allow for classical negation, even though this was not considered in~\cite{nicolas:possibilistic}. However, adding classical negation does not impose any problems and could, as an alternative, easily be simulated in ASP~\cite{baral:knowledge}.

\subsection{Complexity Theory}
Finally, we recall some notions from complexity theory. The complexity classes \spolsig{2} and \spolpi{2} are defined as follows~\cite{papadimitriou:computational}:
\begin{align*}
  \spolsig{0} &= \spolpi{0} = \cP& \\
  \spolsig{1} &= \cNP&\spolsig{2} &= \cNP^\cNP\\
   \spolpi{1} &= \ccoNP & \spolpi{2} &= \cco\spolsig{2}
\end{align*}
where $\cNP^{\cNP}$ is the class of problems that can be solved in polynomial time on a non-deterministic machine with an \cNP\ oracle, \ie assuming a procedure that can solve \cNP\ problems in constant time.
We also consider the complexity class $\complexity{BH}_2$~\cite{cai:boolean}, which is the class of all languages $L$ such that $L = L_1 \cap L_2$, where $L_1$ is in \cNP\ and $L_2$ is in $\ccoNP$.
For a general complexity class $\cC$, a problem is $\cC$-hard if any problem in \cC\ can be polynomially reduced to this problem. 
A problem is said to be \cC-complete if the problem is in \cC\ and the problem is \cC-hard.
Deciding the validity of a Quantified Boolean Formula (QBF) $\phi = \exists X_1 \forall X_2 \cdot p(X_1, X_2)$ with $p(X_1, X_2)$ in disjunctive normal form (DNF) is the canonical $\Sigma_2^P$-complete problem.
 The decision problems we consider in this paper are brave reasoning (deciding whether a literal `$l$' (clause `$e$') is entailed by a consistent answer set of program $P$), cautious reasoning (deciding whether a literal `$l$' (clause `$e$') is entailed by every consistent answer set of a program $P$) and answer set existence (deciding whether a program $P$ has a consistent answer set). Brave reasoning as well as answer set existence for simple, normal and disjunctive programs is $\cP$-complete, $\cNP$-complete and $\spolsig{2}$-complete, respectively~\cite{baral:knowledge}. Cautious reasoning for simple, normal and disjunctive programs is $\cP$-complete, $\ccoNP$-complete and $\spolpi{2}$-complete~\cite{baral:knowledge}.


\section{Characterizing (P)ASP}\label{sec:characterization}
ASP lends itself well to being characterized in terms of modalities. For instance, ASP can be characterized in autoepistemic logic by interpreting `$\naf a$' as the epistemic formula $\neg \mathsf{L}a$ (``$a$ is not believed'')~\cite{gelfond:stratified}. In this paper, as an alternative, we show how ASP can be characterized within possibility theory. 
To arrive at this characterization, we first note that ASP is essentially a special case of PASP in which every rule is certain. As such, we will show how PASP can be characterized within possibility theory. This characterization does not coincide with the semantics proposed in~\cite{nicolas:possibilistic} for PASP, as the semantics from~\cite{nicolas:possibilistic} rely on the classical Gelfond-Lifschitz reduct. Rather, the semantics that we propose for PASP adhere to a different intuition of negation-as-failure. A~characterization of ASP is then obtained from these new semantics by considering the special case in which all rules are entirely certain.

 This characterization of ASP, while still in terms of modalities, stays close in spirit to the Gelfond-Lifschitz reduct. In contrast to the characterization in terms of autoepistemic logic it does not require a special translation of literals to deal with classical negation and disjunction.
 The core idea of our characterization is to encode the meaning of each rule as a constraint on possibility distributions.  Particular minimally specific possibility distributions that satisfy all the constraints imposed by the rules of a program will then correspond to the answer sets of that program.

\noindent In this section, we first limit our scope to possibilistic simple programs (\pointtosec{characterizing:simple}). Afterwards we will broaden the scope and also consider possibilistic normal programs (\pointtosec{characterizing:normal}). The most general case, in which we also consider possibilistic disjunctive programs, will be discussed in \pointtosec{disjunction}.

\subsection{Characterizing Possibilistic Simple Programs}\label{sec:characterizing:simple}
When considering a fact, \ie a rule of the form ${r=(\arule{l_0}{\top})}$, we know by definition that this rule encodes that the literal in the head is necessarily true, \ie $N(l_0) = 1$. If we attach a weight to a fact, then this expresses the knowledge that we are not entirely certain of the conclusion in the head, \ie for a possibilistic rule $p=(r,\lambda)$ we have that $N(l_0) \geq N(\top)$. Note that the constraint uses $\geq$, as there may be other rules in the program that allow us to deduce $l_0$ with a greater certainty.

  In a similar fashion we can characterize a rule of the form (\arule{l_0}{\srange{l}{1}{m}}) as the constraint $N(l_0)\geq N(\srangec{l}{1}{m}{\land})$ which is equivalent to the constraint $N(l_0)\geq \min(\range{N(l_1)}{N(l_m)})$ due to the min-decomposability property of the necessity measure. Indeed, the intuition of such a rule is that the head is only necessarily true when every part of the body is true. When associating a weight with a rule, we obtain the constraint ${N(l_0) \geq \min(\range{N(l_1)}{N(l_m)}, \lambda)}$ for a possibilistic rule $p=(r,\lambda)$ with ${r = (\srule{l})}$.
  Similarly, to characterize a constraint rule, \ie a rule of the form ${r=(\arule{\bot}{\srange{l}{1}{m}})}$, we use the constraint $N(\bot)\geq \min(\range{N(l_1)}{N(l_m)})$, or, in the possibilistic case with $p=(r,\lambda)$, the constraint $N(\bot)\geq \min(\range{N(l_1)}{N(l_m)},\lambda)$.

\begin{definition}\label{def:constraintssimple}
  Let $P$ be a possibilistic simple program and $\function{\pi}{\Omega}{[0,1]}$ a possibility distribution. For every $p \in P$, the constraint $\gamma(p)$ imposed by $p = (r,\lambda)$ with $\lambda \in\ ]0,1]$, ${r = (\srule{l})}$ and $m \geq 0$ is given by 
  \begin{equation}\label{eq:constraints:simple}
    N(l_0) \geq \min(\range{N(l_1)}{N(l_m)}, \lambda).
  \end{equation}
$C_P = \condset{\gamma(p)}{p\in P}$ is the set of constraints imposed by program $P$.  If $\pi$ satisfies the constraints in $C_P$, $\pi$ is said to be a possibilistic model of $C_P$, written $\pi \models C_P$.  A possibilistic model of $C_P$ will also be called a possibilistic model of $P$. We write $S_P$ for the set of all minimally specific possibilistic models of $P$. \end{definition}

\begin{definition}
Let $P$ be a possibilistic simple program. Let $\pi$ be a minimally specific model~of~$P$, \ie $\pi \in S_P$. Then ${V = \condset{\plit{l}{N(l)}}{l \in \hlit{P}}}$ is called a \concept{possibilistic answer set}~of~$P$.
\end{definition}

\begin{example}
Consider the possibilistic simple program $P$ with the rules:
\begin{align*}
\prulealign{0.8}{a}{}&\prulealign{0.6}{\neg b}{a}\\
\prulealign{0.7}{c}{a, \neg b}&\prulealign{0.9}{d}{d}.
\end{align*}
The set $C_P$ consists of the constraints:
\begin{align*}
  N(a) &\geq 0.8 & N(\neg b) &\geq \min(N(a), 0.6)\\
  N(c) &\geq \min(N(a), N(\neg b), 0.7) & N(d) &\geq \min(N(d), 0.9).
\end{align*}
It is easy to see that the last constraint is trivial and can be omitted and that the other constraints can be simplified to ${\Pi(\neg a) \leq 0.2}$, ${\Pi(b) \leq 0.4}$ and ${\Pi(\neg c) \leq 0.4}$. The least specific possibility distribution that satisfies these constraints is given by:
\begin{align*}
  \pi(\set{a,b,c,d}) = 0.4&&\pi(\set{a,c,d}) = 1&&\pi(\set{b,c,d}) = 0.2&&\pi(\set{c,d}) = 0.2\\
  \pi(\set{a,b,c}) = 0.4&&\pi(\set{a,c}) = 1&&\pi(\set{b,c}) = 0.2&&\pi(\set{c}) = 0.2\\
  \pi(\set{a,b,d}) = 0.4&&\pi(\set{a,d}) = 0.4&&\pi(\set{b,d}) = 0.2&&\pi(\set{d}) = 0.2\\
  \pi(\set{a,b}) = 0.4&&\pi(\set{a}) = 0.4&&\pi(\set{b}) = 0.2&&\pi(\set{}) = 0.2.
\end{align*}
By definition, since the possibility distribution $\pi$ satisfies the given constraints, is a possibilistic model. Furthermore, it is easy to see that $\pi$ is the unique minimally specific possibilistic model (due to least specificity). We can verify that ${N(\neg a) = N(b) = N(\neg c) = N(\neg d) = 0}$ since we have that $\pi(\set{a,c,d}) = 1$ and that $N(d) = 0$ since $\pi(\set{a,c}) = 1$. Furthermore it is easy to verify that $N(a) = 0.8$, $N(\neg b) = 0.6$ and $N(c) = 0.6$. Hence we find that $V = \set{\plit{a}{0.8}, \plit{\neg b}{0.6}, \plit{c}{0.6}}$ is a possibilistic answer set of $P$.
\end{example}

\noindent In particular, when we consider all the rules to be entirely certain, \ie $\lambda=1$, the results are compatible with the semantics of classical ASP.

\begin{example}Consider the program $P = \set{(\arule{b}{a}),(\arule{\neg a}{})}$. The set of constraints $C_P$ is given by $N(b) \geq N(a)$ and $N(\neg a) \geq N(\top)$. The first constraint can be rewritten as $1-\Pi(\neg b) \geq 1 - \Pi(\neg a)$, \ie as $\Pi(\neg a) \geq \Pi(\neg b)$. The last constraint can be rewritten as $1-\Pi(a) \geq 1$, \ie as $\Pi(a) = \max\condset{\pi(\omega)}{\omega \models a} = 0$. Given these two constraints, we find that $S_P$ contains exactly one element, which is defined by
\begin{align*}
  \pi(\set{a,b}) = 0&&\pi(\set{a}) = 0\\
  \pi(\set{b}) = 1&&\pi(\set{}) = 1.
\end{align*}
Notice how the first constraint turned out to be of no relevance for this particular example. Indeed, due to the principle of minimal specificity and since there is nothing that prevents $\Pi(\neg a) = 1$, we find that $N(a) = 1 - \Pi(\neg a) = 0$. Therefore the first constraint simplifies to $N(b) \geq 0$. Once more, due to the principle of minimal specificity we thus find that $N(b) = 0$ as there is no information that prevents $\Pi(\neg b) = 1$.
 To find out whether $a, b$, $\neg a$ and $\neg b$ are necessarily true \wrt the least specific possibility distribution $\pi \in S_P$ arising from the program, we verify whether $N(a)=1$, $N(b)=1$, $N(\neg a) = 1$ and $N(\neg b) = 1$, respectively, with $N$ the necessity measure induced by the unique least specific possibility distribution $\pi \in S_P$. As desired, we find that $N(\neg a) = 1 - \Pi(a) = 1$ whereas $N(a) = N(b) = N(\neg b) = 0$. The unique possibilistic answer set is therefore $\set{\plit{\neg a}{1}}$. As we will see, it then follows from \pointtoprop{possiblitydefiniteanswersetA} that the unique classical answer set of $P$ is $\set{\neg a}$.
\end{example}

\noindent In Propositions~\ref{prop:possiblitydefiniteanswersetA}~and~\ref{prop:possiblitydefiniteanswersetB}, below, we prove that this is indeed a correct characterization of simple programs. First, we present a technical lemma.

\begin{lemma}
\label{lem:constraints:equivalent}
Let $L$ be a set of literals, $M \subseteq L$ a consistent set of literals and let the possibility distribution $\pi$ be defined as $\pi(\omega) = 1$ if $\omega \models M$ and $\pi(\omega) =0$ otherwise. Then $M = \condset{l}{N(l) = 1, l \in L}$.
\end{lemma}
\proofat{1--2}

\begin{proposition}
\label{prop:possiblitydefiniteanswersetA}
Let $P$ be a simple program.  If $\pi \in S_P$ then either the unique consistent answer set of $P$ is given by $M = \condset{l}{N(l) = 1, l \in \hlit{P}}$ or $\pi$ is the vacuous distribution, in which case $P$ does not have any consistent answer sets.
\end{proposition}
\proofat{2--4}

\begin{proposition}
\label{prop:possiblitydefiniteanswersetB}
Let $P$ be a simple program.  If $M$ is an answer set of $P$ then the possibility distribution $\pi$ defined by $\pi(\omega) = 1$ iff $\omega \models M$ and $\pi(\omega)=0$ otherwise belongs~to~$S_P$.
\end{proposition}
\proofat{4}

\subsection{Characterizing Possibilistic Normal Programs}\label{sec:characterizing:normal}
To deal with negation-as-failure, 
we rely on a reduct-style approach in which a valuation is guessed and it is verified whether this guess is indeed stable. The approach taken in~\cite{gelfond:stablemodel} to deal with negation-as-failure is to guess an interpretation and verify whether this guess is stable. 
We propose to treat a rule of the form $r=(\grule{l})$ as the constraint \[N(l_0) \geq \min\left(\range{N(l_1)}{N(l_m)}, \range{1- V(l_{m+1})}{1- V(l_n)}\right)\] 
where $V$ is the guess for the valuation and where we assume $min(\set{})=1$. Or, when we consider a possibilistic rule $p=(r,\lambda)$, we treat it as the constraint \[N(l_0) \geq \min\left(\range{N(l_1)}{N(l_m)}, \range{1- V(l_{m+1})}{1- V(l_n)},\lambda\right).\]

We like to make it clear to the reader that the characterization of normal programs in terms of constraints on possibility distributions in its basic form is little more than a reformulation of the Gelfond-Lifschitz approach. The key difference is that this characterization can be used to guess the certainty with which we can derive particular literals from the available rules, rather than guessing what may or may not be derived from it. Nevertheless, this difference plays a crucial role when dealing with uncertain rules. In particular, this characterization of PASP does not coincide with the semantics of~\cite{nicolas:possibilistic} and adheres to a different intuition  for negation-as-failure.

\begin{definition}\label{def:semantics}
Let $P$ be a possibilistic normal program and let $V$ be a valuation. For every $p~\in~P$, the constraint $\gamma_{{}_V}(p)$ induced by $p=(r, \lambda)$ with $\lambda\in\ ]0,1]$, $r = (\grule{l})$ and $V$ is given by
\begin{equation}
  N(l_0) \geq \min\left(\range{N(l_1)}{N(l_m)}, \range{1- V(l_{m+1})}{1- V(l_n)}, \lambda\right)\label{eq:constraint}.
\end{equation}
$C_{(P,V)} = \condset{\gamma_{{}_V}(p)}{p\in P}$ is the set of constraints imposed by program $P$ and valuation $V$, and $S_{(P,V)}$ is the set of all minimally specific possibilistic models of $C_{(P,V)}$.
\end{definition}

\begin{definition}\label{def:possibilisticanswerset}
  Let $P$ be a possibilistic normal program and let $V$ be a valuation. Let $\pi \in S_{(P,V)}$ be such that
  \begin{equation*}
  \forall l \in \hlit{P} \cdot N(l) = V(l)
    \end{equation*}
  then $V = \condset{\plit{l}{N(l)}}{l \in \hlit{P}}$ is called a possibilistic answer set of $P$.
\end{definition}

\begin{example}\label{ex:airport:again}
Consider the possibilistic normal program $P$ from \pointtoex{abnormal}. The constraints $C_P$ induced by $P$ are:
\begin{align*}
N(\mathit{invalid}) &\geq 0.1\\
N(\mathit{airport}) &\geq \min(1-V(\mathit{invalid}), 1)
\end{align*}
From the first constraint it readily follows that we need to choose ${V(\mathit{invalid}) = 0.1}$ to comply with the principle of minimal specificity. The other constraint can then readily be simplified to:
\begin{align*}
N(\mathit{airport}) &\geq 0.9
\end{align*}
Hence it follows that $V = \set{\plit{invalid}{0.1}, \plit{airport}{0.9}}$ is the unique possibilistic answer set of $P$.
\end{example}

It is easy to see that the proposed semantics remain closer to the intuition of the possibilistic normal program discussed in the introduction. Indeed, we conclude with a high certainty that we need to go to the airport.

Still, it is interesting to further investigate the particular relationship between the semantics for PASP as proposed in~\cite{nicolas:possibilistic} and the semantics presented in this section. Let the possibilistic rule $r$ be of the form: \[ \pdrule{l}. \]
When we determine the reduct \wrt a~valuation~$V$ of the possibilistic program containing $r$, then the certainty of the rule  in the reduct that corresponds with $r$ can be verified to be: \[\min(\range{F_N(V(l_{m+1}))}{F_N(V(l_n))},\lambda)\] with $F_N$ a fuzzy negator, \ie where $F_N$ is a decreasing function with ${F_N(0) = 1}$ and ${F_N(1) = 0}$. In particular, for the semantics of~\cite{nicolas:possibilistic} we have that $F_N$ is the G\"{o}del negator~$F_\text{G}$, defined as $F_\text{G}(0) = 1$ and $F_\text{G}(c) = 0$ with $0 < c \leq 1$. In~the~semantics for PASP presented in this section, $F_N$ is the \L ukasiewicz negator $F_{\text{\L}}(c) = 1-c$ with $0 \leq c \leq 1$. Thus, for a rule such as:
\[
  \prule{0.9}{b}{\naf a}
\]
and a valuation $V=\set{\plit{a}{0.2}}$ we obtain under the approach from~\cite{nicolas:possibilistic} the reduct $(\prule{0}{b}{})$, whereas under our approach we obtain the constraint $N(b) \geq \min(0.9, 1-0.2)$, which can be encoded by the rule $(\prule{0.8}{b}{})$. Essentially, the difference between both semantics can thus be reduced to a difference in the choice of negator. However, even though the semantics share similarities, there is a notable difference in the underlying intuition of both approaches. Specifically, in the semantics presented in this paper, we have that `$\naf l$' is understood as ``\emph{the degree to which `$\neg l$' is possible}'', or, equivalently, ``\emph{the degree to which it is not the case that we can derive `$l$' with certainty}''. This contrasts with the intuition of `$\naf l$' in~\cite{nicolas:possibilistic} as a Boolean condition and understood as ``\emph{we cannot derive `$l$' with a strictly positive certainty}''.

Interestingly, we find that the complexity of the main reasoning tasks for possibilistic normal programs remains at the same level of the polynomial hierarchy as the corresponding normal ASP programs.

While we will see in \pointtosec{complexity} that the complexity of possibilistic normal programs remains unchanged compared to classical normal programs, it is important to note that under the semantics proposed in this section there is no longer a 1-on-1 mapping between the classical answer sets of a normal program and the possibilistic answer sets. Indeed, if we consider a possibilistic normal program constructed from a classical normal program where we attach certainty $\lambda=1$ to each rule, then we can sometimes obtain additional intermediary answer sets. 
Consider the next example:

\begin{example}\label{ex:extras}
  Consider the normal program with the single rule $\arule{a}{\naf a}$. This program has no classical answer sets. Now consider the possibilistic normal program $P$ with the rule
  \begin{equation*}
    \prule{1}{a}{\naf a}.
  \end{equation*}
  The set of constraints $C_{(P, V)}$ is given by
  \begin{equation*}
    N(a) \geq \min(1-V(a), 1).
  \end{equation*}
This constraint can be rewritten as 
\begin{align*}
 &N(a) \geq \min(1-V(a), 1)\\
 &\equiv N(a) \geq 1-V(a)\\
 & \equiv 1-\Pi(\neg a) \geq 1-V(a)\\
 &\equiv \Pi(\neg a) \leq V(a).
\end{align*}
We thus find that the set $S_{(P, V)}$ is a singleton with $\pi \in S_{(P, V)}$ defined by $\pi(\set{a}) = 1$  and ${\pi(\set{}) = V(a)}$.
We can now establish for which choices of $V(a)$ it holds that $V(a) = N(a)$:
\begin{align*}
  V(a) &= N(a)\\
  \Pi(\neg a) &= 1-\Pi(\neg a)\\
  2\cdot\Pi(\neg a) &= 1
\end{align*}
and thus, since $\Pi(\neg a) \leq V(a)$, we have $\pi(\set{}) = 0.5$. The unique possibilistic answer set of $P$ is therefore \set{\plit{a}{0.5}}. In the same way, one may verify that the program
\begin{align*}
  \parule{1}{a}{\naf b}&\parule{1}{b}{\naf a}
\end{align*}
has an infinite number of possibilistic answer sets, i.e. $\set{a^c, b^{1-c}}$ for every ${c \in [0,1]}$. For practical purposes, however, this behavior has a limited impact as we only need to consider a finite number of certainty levels to perform brave/cautious reasoning. 
Indeed, we only need to consider the  certainties used in the program, their complement to account for negation-as-failure and $\frac{1}{2}$ to account for the intermediary value as in \pointtoex{extras}. Thus, for the main reasoning tasks it suffices to limit our attention to the certainties from the set $cert^+(P)$.
\end{example}

We now show that when we consider rules with an absolute certainty, \ie classical normal programs, we obtain a correct characterization of classical ASP, provided that we restrict ourselves to absolutely certain conclusions, \ie valuations $V$ for which it holds that $\forall l \cdot V(l) \in \set{0,1}$.

\begin{example}\label{ex:constraints}
Consider the program $P$ with the rules
\begin{align*}
  \arule{a&}{} & \arule{b&}{b} & \arule{c&}{a, \naf b}.
\intertext{The set of constraints $C_{(P,V)}$ is then given by}
  N(a) &\geq 1&
  N(b) &\geq N(b)&
  N(c) &\geq \min\left(N(a), 1-V(b)\right).
\intertext{
We can rewrite the first constraint as $1-\Pi(\neg a) \geq 1$ and thus $\Pi(\neg a) = 0$. The second constraint is trivially satisfied and, since it does not entail any new information, can be dropped. The last constraint can be rewritten as $\Pi(\neg c) \leq 1 - \min(1-\Pi(\neg a), 1-V(b))$, which imposes an upper bound on the value that $\Pi(\neg c)$ can assume. Since we already know that $\Pi(\neg a) =0$ we can further simplify this inequality to $\Pi(\neg c) \leq 1 - \min(1-0, 1-V(b)) = 1- (1-V(b)) = V(b)$. In conclusion, the program imposes the constraints
}
  \Pi(\neg a)&=0 &&& \Pi(\neg c)&\leq V(b).
\end{align*}
The set $S_{(P,V)}$ then contains exactly one element, which is defined by
\begin{align*}
  \pi(\set{a,b,c}) &= 1 & \pi(\set{b,c}) &= 0\\
  \pi(\set{a,b}) &= V(b) & \pi(\set{b}) &= 0\\
  \pi(\set{a,c}) &= 1 & \pi(\set{c}) &= 0\\
  \pi(\set{a}) &= V(b) & \pi(\set{}) &= 0.
\end{align*}
Note that this possibility distribution is independent of the choice for $V(a)$ and $V(c)$ since there are no occurrences of `$\naf a$' and `$\naf c$' in $P$. It remains then to determine for which choices of $V(b)$ it holds that $V(b) = N(b)$, \ie for which the guess $V(b)$ is stable. We have:
$$
V(b) = N(b) = 1 - \Pi(\neg b) = 1 - \max\condset{\pi(\omega)}{\omega \models \neg b} = 0
$$
and thus we find that $\pi(\set{a,b}) = \pi(\set{a}) = 0$. We have $N(a) = 1 - \Pi(\neg a) = 1$, $N(c) = 1 - \Pi(\neg c) = 1$ and $N(b) = 1 - \Pi(\neg b) = 0$. As we will see in the next propositions, the unique answer set of $P$ is therefore $\set{a,c}$.
\end{example}

\begin{proposition}
\label{prop:possiblitynormalanswerset}
Let $P$ be a normal program and $V$ a valuation. Let $\pi \in S_{(P,V)}$ be such that 
\begin{align}
  &\sForall{l \in \hlit{P}}{V(l) = N(l)}\label{rule:stable}
  \text{ ; and}\\
  &\sForall{l \in \hlit{P}}{N(l) \in \set{0,1}}\label{rule:consistent}
\end{align}
then  $M = \condset{l}{N(l) = 1, l \in \hlit{P}}$ is an answer set of the normal program $P$. 
\end{proposition}

\begin{proof}
 This proposition is a special case of \pointtoprop{disjunctivepossibilistic} presented below.
\end{proof}

Note that the requirement stated in \eqref{rule:consistent} cannot be omitted.  Let us consider \pointtoex{extras}, in which we considered the normal program $P=\set{\arule{a}{\naf a}}$. This normal program $P$ has no classical answer sets. The constraint that corresponds with the rule $(\arule{a}{\naf a})$ is $N(a) \geq 1- V(a)$. For a choice of $V = \set{a^{0.5}}$, however, we would find that $V(a) = N(a)$ and thus that $V$ is an answer set of $P$ if we were to omit this requirement.

\begin{proposition}
\label{prop:normalanswersetpossibility}
Let $P$ be a normal program. If $M$ is an answer set of $P$, there is a valuation $V$, defined by $V(l) = 1$ if $l \in M$ and $V(l) = 0$ otherwise, and a possibility distribution $\pi\in S_{(P,V)}$ such that for every $l \in \hlit{P}$ we have $V(l) = N(l)$ (\ie $N(l) = 1$ if $l \in M$ and $N(l) = 0$ otherwise).
\end{proposition}

\begin{proof}
 This proposition is a special case of \pointtoprop{possibilisticdisjunctive} presented below.
\end{proof}

We like to point out to the reader that we could try to encode the information in a rule in such a way that we interpret `$\naf a$' as $\Pi(\neg a)$, which closely corresponds to the intuition of negation-as-failure. Indeed, when it is completely possible to assume that `$\neg a$' is true, then surely `$\naf a$' is true. Under this encoding, however, we run into a significant problem. Consider the rules ($\arule{b}{\naf c}$) and ($\arule{c}{\naf b}$). These rules would then correspond with the constraints $N(b) \geq \Pi(\neg c)$ and $N(c) \geq \Pi(\neg b)$, respectively. Notice though that both constraints can be rewritten as the constraint $1 - \Pi(\neg b) \geq \Pi(\neg c)$. This would imply that both rules are semantically equivalent in ASP, which is clearly not the case. Hence we cannot directly encode `$\naf a$' as $\Pi(\neg a)$ and guessing a valuation is indeed necessary since without the guess $V$ we would not be able to obtain a unique set of constraints. As we have shown this only affects literals preceded by negation-as-failure and we can continue to interpret a literal `$b$' as $N(b)$.


\section{Possibilistic Semantics of Disjunctive ASP Programs}\label{sec:disjunction}

We now turn our attention to how we can characterize disjunctive rules. We found in \pointtosec{characterization} that we can characterize a rule of the form $r=(\arule{head}{body})$ as the constraint $N(head) \geq N(body)$, or, similarly, that we can characterize a possibilistic  rule $p=(r,\lambda)$ as the constraint $N(head) \geq \min(N(body),\lambda)$. 
Such a characterization works particularly well due the min-decomposability \wrt conjunction. Indeed, since the body of \eg a simple rule $r = (\srule{l})$ is a conjunction of literals we can write ${body = \srangec{l}{1}{m}{\land}}$. Then $N(body)$ can be rewritten as $\min(\range{N(l_1)}{N(l_m)})$, which allows for a straightforward simplification. In a similar fashion, for a positive disjunctive rule ${r = (\arule{\srangec{l}{0}{k}{;}}{\srange{l}{k+1}{m}})}$ we can readily write $N(body)$ as $\min(\range{N(l_{k+1})}{N(l_m)})$. We would furthermore like to simplify $N(head)$ with $head = \srangec{l}{0}{k}{\lor}$. However, we do not have that $N(head) = \max(\range{N(l_0)}{N(l_k)})$. Indeed, in general we only have that $N(head) \geq \max(\range{N(l_0)}{N(l_k)})$.
This means that we can either choose to interpret the head as $\max(\range{N(l_0)}{N(l_k)})$ or $N(\srangec{l}{0}{k}{\vee})$. In particular, a \concept{possibilistic disjunctive rule} $p=(r,\lambda)$ with 
\[r=(\drule{l})\]
can either be interpreted as the constraint 
\begin{align}
  \max(\range{N(l_0)}{N(l_k)}) \geq \min(\range{N(l_{k+1})}{N(l_m)}, \prange{1-}{V(l_{m+1})}{V(l_n)},\lambda)\label{eq:strong}
\end{align}
which we will call the \concept{strong} interpretation of disjunction, or as the constraint
\begin{align}
  N(\srangec{l}{0}{k}{\vee}) \geq \min(\range{N(l_{k+1})}{N(l_m)}, \prange{1-}{V(l_{m+1})}{V(l_n)},\lambda)\label{eq:weak}
\end{align}
which we will call the \concept{weak} interpretation of disjunction. In the remainder of this paper, we syntactically differentiate between both approaches by using the notation $\srangec{l}{0}{k}{;}$ and $\srangec{l}{0}{k}{\lor}$ to denote the strong and the weak interpretation of disjunction, respectively.

The choice of how to treat disjunction is an important one that crucially impacts the nature of the resulting answer sets.
For example, the non-deterministic nature of strong disjunction provides a useful way to generate different (candidate) solutions, whereas weak disjunction is oftentimes better suited when we are interested in modelling the epistemic state of an agent since it amounts to accepting the disjunction as being true rather than making a choice of which disjunct to accept. In this section we consider both characterizations; the characterization of disjunction as \eqref{eq:strong} is discussed in \pointtosec{disjunction:strong} and in \pointtosec{disjunction:weak} we discuss the characterization as \eqref{eq:weak}. In particular we will show that the first characterization of disjunction corresponds to the semantics of disjunction found in ASP whereas the Boolean counterpart of the second characterization has, to the best of our knowledge, not yet been studied in the literature.


\subsection{Strong Possibilistic Semantics of Disjunctive Rules}\label{sec:disjunction:strong}

We first consider the characterization of disjunction in which we treat a disjunction of the form `$\srangec{l}{0}{k}{;}$' as $\max(N(l_0), \ldots, N(l_k))$. As it turns out, under these strong possibilistic semantics the disjunction behaves as in classical ASP.

\begin{definition}\label{def:strong}
  Let $P$ be a possibilistic disjunctive program and let $V$ be a valuation. For every possibilistic disjunctive rule $p=(r,\lambda)$ with $\lambda \in\ ]0,1]$ and $r = (\drule{l})$ the constraint $\gamma^\mathrm{s}_{{}_V}(p)$ induced by $p$ and $V$ is given by
  \begin{equation}
    \max(\range{N(l_0)}{N(l_k)}) \geq \min(\range{N(l_{k+1})}{N(l_m)}, \range{1-V(l_{m+1})}{1-V(l_n)},\lambda) \label{disjunction:strong}
  \end{equation}
  $C^\mathrm{s}_{(P,V)} = \condset{\gamma^\mathrm{s}_{{}_V}(p)}{p \in P}$ is the set of constraints imposed by program $P$ and $V$, and $S^\mathrm{s}_{(P,V)}$ is the set of all minimally specific possibilistic models of $C^\mathrm{s}_{(P,V)}$.\footnote{We use the superscript `s' to highlight that we employ the semantics of \emph{s}trong disjunction.}
\end{definition}
Whenever $P$ is a positive disjunctive program, \ie whenever $P$ is a disjunctive program without negation-as-failure, \eqref{disjunction:strong} is independent of $V$ and we simplify the notation to $\gamma^\mathrm{s}, C^\mathrm{s}_P$ and $S^\mathrm{s}_P$.

Notice that, unlike in possibilistic logic where a unique least specific possibility distribution exists because of the specific form of the considered constraints, the constraint of the form \eqref{disjunction:strong} can give rise to multiple minimally specific possibility distributions of which some will correspond with answer sets. Indeed, the program $P = \set{\arule{a;b}{}}$ induces the constraint $\max(N(a),N(b)) \geq 1$, which has two minimally specific possibility distributions, yet no least specific possibility distribution. Indeed, we have the minimally specific possibility distributions $\pi_1, \pi_2$ defined by
\begin{align*}
  \pi_1(\set{a,b}) &= 1 & \pi_1(\set{b}) &= 0&&&\pi_2(\set{a,b}) &= 1 & \pi_2(\set{b}) &= 1\\
  \pi_1(\set{a}) &= 1 & \pi_1(\set{}) &= 0&&&\pi_2(\set{a}) &= 0 & \pi_2(\set{}) &= 0
\end{align*}

\begin{definition}\label{def:possibilisticanswersetdisjunctive}
  Let $P$ be a possibilistic disjunctive program and let $V$ be a valuation. Let $\pi \in S^\mathrm{s}_{(P,V)}$ be such that
  \begin{equation*}
  \forall l \in \hlit{P} \cdot N(l) = V(l)
    \end{equation*}
  then $V = \condset{\plit{l}{N(l)}}{l \in \hlit{P}}$ is called a possibilistic answer set of $P$.
\end{definition}

We now further illustrate the semantics and the underlying intuition by considering a possibilistic disjunctive program in detail.

\begin{example}
Consider the possibilistic (positive) disjunctive program $P$ with the following rules:
\begin{align*}
\parule{0.8}{a;b}{}\\
\parule{0.6}{c}{a}\\
\parule{0.4}{c}{b}.
\end{align*}
The constraints $C^\mathrm{s}_P$ induced by this program are:
\begin{align*}
\max(N(a), N(b)) &\geq 0.8\\
N(c) &\geq \min(N(a), 0.6)\\
N(c) &\geq \min(N(b), 0.4).
\end{align*}
From the first constraint it follows that we either need to choose $V(a) = 0.8$ or $V(b) = 0.8$, in accordance with the principal of minimal specificity. Hence, we either obtain $V(c) = 0.6$ or $V(c) = 0.4$. As such we find that the two unique possibilistic answer sets of $P$ are $\set{\plit{a}{0.8}, \plit{c}{0.6}}$ and $\set{\plit{b}{0.8}, \plit{c}{0.4}}$.
\end{example}

\noindent As before, if we restrict ourselves to rules that are entirely certain we obtain a characterization of disjunctive programs in classical ASP.

\begin{example}\label{ex:strong}
Consider the program $P$ with the rules
\begin{align*}
  \arule{a;b&}{} & \arule{a&}{b}
\intertext{The set of constraints $C^\mathrm{s}_{P}$ is given by}
  \max(N(a), N(b)) &\geq N(\top) = 1&
  N(a) &\geq N(b).
\end{align*}
Intuitively, the first constraint induces a choice. To satisfy this constraint, we need to take either $N(a) = 1$ or $N(b) = 1$. Depending on our choice, we can consider two possibility distributions. The possibility distribution $\pi_1$ is the least specific possibility distribution that satisfies the constraints $N(a) = 1$ and $N(a) \geq N(b)$, whereas $\pi_2$ is the least specific possibility distribution satisfying the constraints $N(b) = 1$ and $N(a) \geq N(b)$:
\begin{align*}
  \pi_1(\set{a,b}) &= 1 & \pi_1(\set{b}) &= 0\\
  \pi_1(\set{a}) &= 1 & \pi_1(\set{}) &= 0
\end{align*}
and
\begin{align*}
  \pi_2(\set{a,b}) &= 1 & \pi_2(\set{b}) &= 0\\
  \pi_2(\set{a}) &= 0 & \pi_2(\set{}) &= 0.
\end{align*}
It is clear that the possibility distribution $\pi_2$ cannot be minimally specific \wrt the constraints ${\max(N(a), N(b)) = 1}$ and $N(a) \geq N(b)$ since $\pi_1(\set{a}) > \pi_2(\set{a})$ and ${\pi_1(\omega) \geq \pi_2(\omega)}$ for all other interpretations $\omega$. We thus have that $S^\mathrm{s}_P$ only contains a single element, namely $\pi_1$. With $N$ the necessity measure induced by $\pi_1$ we obtain $N(a) = 1$ and $N(b) = 0$. As will follow from \pointtoprop{disjunctivepossibilistic} and \ref{prop:possibilisticdisjunctive} the unique answer set of $P$ is therefore $\set{a}$.

\noindent Let us now add the rule (\arule{b}{\naf b}) to $P$. Notice that in classical ASP this extended program has no answer sets. The set of constraints $C^\mathrm{s}_{(P,V)}$ is given by:
\[
C^\mathrm{s}_P \cup \set{N(b) \geq 1-V(b)}.
\]
This new constraint, intuitively, tells us that `$b$' must necessarily be true, since we force it to be true whenever it is not true. Note, however, that the act of making `$b$' true effectively removes the motivation for making it true in the first place. As~expected, we cannot find any minimally specific possibilistic model that agrees with the constraints imposed by $P$ and $V$ such that $\forall l \in \hlit{P} \cdot N(l) \in \set{0,1}$. The problem has to do with our choice of $V(b)$. If we take $V(b) = 1$ then the constraint imposed by the first rule still forces us to choose either $N(a) = 1$ or $N(b) = N(a) = 1$ due to the interplay with the constraint imposed by the second rule. However, $S^\mathrm{s}_{(P,V)}$ contains only one minimally specific possibility distribution, namely the one with $N(a) = 1$. Hence $N(b) = 0 \neq V(b)$. If we take $V(b) = 0$ then the last rule forces $N(b) = 1$. Hence $V(b) = 0 \neq 1 = N(b)$.
\end{example}

\noindent Now that we have clarified the intuition, we can formalize the connection between the strong possibilistic semantics and classical disjunctive ASP.

\begin{proposition}
\label{prop:disjunctivepossibilistic}
Let $P$ be a disjunctive program, $V$ a valuation and let $\pi \in S^\mathrm{s}_{(P,V)}$ be~such~that
\begin{align}
\forall l \in \hlit{P} &\cdot V(l) = N(l)\label{eq:disstable}\text{ ; and}\\
\forall l \in \hlit{P} &\cdot N(l) \in \set{0,1}\label{eq:disconsistent}
\end{align}
then $M = \condset{l}{N(l) = 1, l \in \hlit{P}}$ is an answer set of the disjunctive program~$P$.
\end{proposition}
\proofat{4--5}

\begin{proposition}
\label{prop:possibilisticdisjunctive}
Let $P$ be a disjunctive program. If $M$ is an answer set of $P$, there is a valuation $V$, defined as $V(l) = 1$ if $l \in M$ and $V(l) = 0$ otherwise, and a possibility distribution~$\pi$, defined as $\pi(\omega) = 1$ if $\omega \models M$ and $\pi(\omega) = 0$ otherwise, such that $\pi \in S^\mathrm{s}_{(P,V)}$ and for every $l \in \hlit{P}$ we have $V(l) = N(l)$.
\end{proposition}
\proofat{5--6}


\subsection{Weak Possibilistic Semantics of Disjunctive Rules}\label{sec:disjunction:weak}
Under the strong possibilistic semantics of disjunction we consider all the disjuncts of a satisfied rule separately. Under this non-deterministic view the rule (\arule{a;b}{}) means that `$a$' is believed to be true or `$b$' is believed to be true. When looking at answer sets as epistemic states it becomes apparent that there is also another choice in how we can treat disjunction in the head. Indeed, we can look at the disjunction as a whole to hold, without making any explicit choices as to which of the disjuncts holds. When trying to reason about one's knowledge there are indeed situations in which we do not want, or simply cannot make, a choice as to which of the disjuncts is true. This implies that we need to look at an answer set as a set of clauses, rather than a set of literals.

An elaborate example using weak disjunction and uncertainty has been given in \pointtosec{introduction}. In this subsection we consider the semantics of such programs. For~starters, we will extend the PASP semantics with the notion of clauses, rather than literals, and define an applicable immediate consequence operator for programs composed of clauses. We then prove some important properties, such as the monotonicity of the immediate consequence operator. For the classical case (\ie when omitting weights), we furthermore characterize the complexity of clausal programs, both with and without negation-as-failure in \pointtosec{complexity}. In particular, we show how the complexity is critically determined by whether we restrict ourselves to atoms and highlight, as shown by the higher complexity of some of the reasoning tasks, that weak disjunction is a non-trivial extension of ASP.

We start by formally defining possibilistic clausal programs, \ie possibilistic programs with a syntax that allows for disjunction in the body. We then define the weak possibilistic semantics of such clausal programs in terms of constraints on possibility distributions. 
We also introduce an equivalent characterization based on an immediate consequence operator and a reduct, which is more in line with the usual treatment of ASP programs. When all the rules are entirely certain we obtain the classical counterpart, which we name clausal programs. 

\subsubsection{Semantical Characterization}\label{sec:weak:definition}
We rely on the notion of a \concept{clause}, \ie a~finite disjunction of literals. Consistency and entailment for sets of clauses are defined as in propositional logic. As~such, we can derive from the information `$a \lor b \lor c$' and `$\neg b$' that `$a \lor c$' is true.

\begin{definition}
 A~\concept{clausal rule} is an expression of the form (\grule{e}) 
 with $e_i$ a clause for every $0 \leq i \leq n$. A~\concept{positive clausal rule} is an expression of the form (\srule{e})
, \ie a~clausal rule without negation-as-failure. A \concept{(positive) clausal program} is a finite set of (positive) clausal rules.
\end{definition}

For a clausal rule, which is of the form $r = (\grule{e})$, we say that $e_0$ is the \concept{head} and that $\gbody{e}{1}{m}{n}$ is the \concept{body} of the clausal rule. We use the notation $head(r)$ and $body(r)$ to denote the clause in the head, \resp the set of clauses in the body. The Herbrand base $\hbase{P}$ of a clausal program $P$ is still defined as the set of atoms appearing in $P$. As such, possibility distributions are defined in the usual way as $\function{\pi}{2^{\hbase{P}}}{[0,1]}$ mappings.

Until now, we were able to define the possibility distributions that satisfied the constraints imposed by the rules in a program in terms of a valuation $V$, \ie a $\function{V}{\hlit{P}}{[0,1]}$ mapping. This need no longer be the case. Specifically, note that we will now impose constraints of the form $N(\srangec{l}{0}{k}{\lor}) \geq \lambda$. Assume that we have a possibility distribution $\pi$ defined as
\begin{align*}
  \pi(\set{a,b,c}) &= 0 &\pi(\set{a,b}) &= 0 &\pi(\set{a,c}) &= 1 &\pi(\set{a}) &= 1\\
  \pi(\set{b,c}) &= 0 &\pi(\set{b}) &= 0 &\pi(\set{c}) &= 1 &\pi(\set{}) &= 0.
\end{align*}
This possibility distribution is the least specific possibility distribution that satisfies the constraints $N(a \lor b \lor c) = 1$ and $N(\neg b) = 1$. However, it can be verified that this possibility distribution cannot be defined in terms of a mapping ${\function{V}{\hlit{P}}{[0,1]}}$. 

Instead, we~define the set of clauses appearing in the head of the rules of a clausal program $P$ as $\hclause{P} = \condset{head(r)}{r \in P}$. Given a clausal program, it is clear that the only information that can be derived from the program are those clauses that are in the head of a rule. To compactly describe a possibility distribution imposed by clausal programs we will thus, for the remainder of this section and for \pointtosec{complexity}, take a valuation $V$ to be a $\hclause{P} \rightarrow [0,1]$ mapping. As~before, a valuation $V$ corresponds with the set of constraints $\condset{N(e) \geq \lambda}{\plit{e}{\lambda} \in V}$. The set notation for valuations and the notations $V^\lambda$~and~$V^{\underline{\lambda}}$ are extended as usual. 
Entailment for valuations is defined as in possibilistic logic, \ie if we consider the least specific possibility distribution $\pi_V$ satisfying the constraints $\condset{N_V(e) \geq \lambda}{\plit{e}{\lambda} \in V}$ then ${V \models p^\lambda}$ with `$p$' a proposition iff $N_V(p) \geq \lambda$. In particular, recall from possibilistic logic the inference rules (GMP) or graded modus ponens, \ie we can infer from $N(\alpha) \geq \lambda$ and $N(\alpha \rightarrow \beta) \geq \lambda'$ that $N(\beta) \geq \min(\lambda, \lambda')$. In addition recall the inference rule (S), \ie we can infer from $N(\alpha) \geq \lambda$ that $N(\alpha) \geq \lambda'$ with $\lambda \geq \lambda'$.

\begin{definition}
A~\concept{possibilistic (positive) clausal program} is a set of possibilistic (positive) clausal rules, which are pairs $p = (r, \lambda)$ with $r$ a (positive) clausal rule and $\lambda \in\ ]0,1]$ a certainty associated with $r$.
\end{definition}
We define $P^*$ and the $\lambda$-cut $P_\lambda$ as usual.

We are now almost able to define the semantics of weak disjunction. In the previous sections we guessed a valuation and used this valuation to deal with negation-as-failure. However, for clausal programs, a new problem arises. Note that the least specific possibility distribution that satisfies the constraints $N(a \lor b \lor c) = 1$ and $N(\neg b) = 1$ is also the least specific possibility distribution that satisfies the constraints $N(a \lor c)$ and $N(\neg b)$. As~such, if $\hclause{P} = \set{(a \lor b \lor c), (\neg b), (a \lor c)}$, there would not be a unique valuation that can be used to define this least specific possibility distribution. Indeed, a valuation uniquely defines a possibility distribution, but not vice versa. To avoid such ambiguity, we will instead immediately guess a possibility distribution $\pi_V$ and use this possibility distribution to deal with negation-as-failure in a clausal program.

\begin{definition}\label{def:weak}
  Let $P$ be a possibilistic clausal program and let $\pi_V$ be a possibility distribution. For every $p \in P$, the constraint $\gamma^\mathrm{w}_{\pi_V}(p)$ induced by $p=(r,\lambda)$ with $\lambda\in\ ]0,1]$, $r = (\grule{e})$ and $\pi_V$ under the weak possibilistic semantics is given~by
  \begin{equation}
    N(e_0) \geq \min(\range{N(e_1)}{N(e_m)}, \prange{1-}{N_V(e_{m+1})}{N_V(e_n)},\lambda)\label{disjunction:weak}.
  \end{equation}
  $C^\mathrm{w}_{(P,\pi_V)} = \condset{\gamma^\mathrm{w}_{\pi_V}(p)}{p \in P}$ is the set of constraints imposed by program $P$ and~$\pi_V$, and $S^\mathrm{w}_{(P,\pi_V)}$ is the set of all minimally specific possibilistic models of $C^\mathrm{w}_{(P,\pi_V)}$. 
\end{definition}

Whenever $P$ is a possibilistic (positive) clausal program, \ie whenever $P$ is a possibilistic clausal program without negation-as-failure, \eqref{disjunction:weak} is independent of~$\pi_V$ and we simplify the notation to $\gamma^\mathrm{w}, C^\mathrm{w}_P$ and $S^\mathrm{w}_P$.

\vspace{5px}
\begin{definition}\label{def:weak:poss:semantics}
Let $P$ be a possibilistic clausal program. Let $\pi_V$ be a possibility distribution such that $\pi_V \in S^\mathrm{w}_{(P,\pi_V)}$. We then say that $\pi_V$ is a possibilistic answer set of~$P$.
\end{definition}
As already indicated we can also use a valuation $V$ to concisely describe $\pi_V$. When we say that $V$ is a possibilistic answer set of the clausal program~$P$ we are, more precisely, stating that the possibility distribution induced by $V$ is a possibilistic answer set of the clausal program~$P$.

\begin{lemma}\label{lem:singleclausal}
Let $P$ be a possibilistic positive clausal program. Then $S^\mathrm{w}_{(P,\pi_V)}$ is a singleton, \ie ${\pi \in S^\mathrm{w}_{(P,\pi_V)}}$ is a least specific possibility distribution.
\end{lemma}
\begin{proof}
This readily follows from the form of the constraints imposed by the rules $p \in P$ and since a possibilistic positive clausal program is free of negation-as-failure.
\end{proof}

\begin{example}\label{ex:clausal:example}
Consider the possibilistic clausal program $P$ with the rules:
\begin{align*}
  \prulealign{1}{a\lor c \lor d}{}\\
  \prulealign{0.4}{\neg d}{}\\
  \prulealign{0.8}{e}{\naf (a \lor b \lor c)}.
\end{align*}

We have that $C^{\mathrm{w}}_{(P,\pi_V)}$ is the set of constraints:
\begin{align*}
  N(a\lor c \lor d) &\geq 1\\
  N(\neg d) &\geq 0.4\\
  N(e) &\geq \min(1-N_V(a \lor b \lor c), 0.8).
\end{align*}
We can rewrite the first constraint as $N(\neg d \rightarrow a \lor c) \geq 1$. Given the second constraint $N(\neg d) \geq 0.4$ we can apply the inference rule (GMP) to conclude that $N(a \lor c) \geq 0.4$. From propositional logic we know that $(a \lor c) \rightarrow (a \lor b \lor c)$, \ie we also have $N(a \lor b \lor c) \geq 0.4$. 

For $\pi_V$ to be an answer set of $P$ we know from \pointtodef{weak:poss:semantics} that we must have that $\pi \in S^{\mathrm{w}}_{(P,\pi_V)}$ with $\pi = \pi_V$. In other words, we must have that ${N_V(a \lor b \lor c)} = {N(a \lor b \lor c)} \geq 0.4$. Due to the principle of least specificity, which implies that ${N(a \lor b \lor c)} = 0.4$, the last constraints can be simplified to ${N(e) \geq \min(1-0.4, 0.8)}$ or $N(e) \geq 0.6$. As~such, the least specific possibility distribution defined by the constraints $N(e) \geq 0.6$, $N(a\lor c \lor d) \geq 1$ and $N(\neg d) \geq 0.4$ is a possibilistic answer set of $P$.
\end{example}

Notice that we implicitly defined the possibilistic answer set of the previous example as a valuation, \ie in terms of clauses that appear in the head. Alternatively we could thus write that $V = \set{\plit{e}{0.6}, \plit{a \lor b \lor d}{1}, \plit{\neg b}{0.4}}$ defines the possibilistic answer set of $P$. This idea will be further developed in \pointtosec{weak:syntactic} to avoid the need to explicitly define a possibility distribution (which would require an exponential amount of space) and instead rely on an encoding of a possibility distribution by a (polynomial) set of weighted clauses.

For the crisp case, we only want clauses that are either entirely certain or completely uncertain, \ie true or false. To this end, we add the constraint \eqref{eq:clauseconsistent}, which is similar to \eqref{rule:consistent} from \pointtoprop{possiblitynormalanswerset}.

\begin{definition}\label{def:weak:semantics}
Let $P$ be a clausal program and $\pi_V \in S^\mathrm{w}_{(P,\pi_V)}$ a possibility distribution such that
\begin{align}
\forall \omega \in \Omega &\cdot \pi_V(\omega) \in \set{0,1}\label{eq:clauseconsistent}
\end{align}
then $\pi_V
$ is called an answer set of~$P$.
\end{definition}

\subsubsection{Syntactic Characterization}\label{sec:weak:syntactic}

We now introduce a syntactic counterpart of the semantics for weak disjunction by defining an immediate consequence and reduct operator. As such, it is more in line with the classical Gelfond-Lifschitz approach. In addition, the syntactic approach only needs polynomial size (as we will only consider clauses appearing in the head of the clausal rules). Indeed, what we will do is formalise the idea of using a valuation to determine the possibilistic answer sets of a clausal program, rather than relying on an exponential possibility distribution.

\begin{definition}\label{def:weak:immediate}
Let $P$ be a possibilistic positive clausal program. We define the immediate consequence operator~$T^{\mathrm{w}}_P$~as:
\begin{align*}
T^\mathrm{w}_P(V)(e_0) = \max\condset{\lambda \in [0,1]}{(\srule{e}) \in P_\lambda \text{ and } \sForall{i \in \set{\range{1}{m}}}{V^\lambda \models e_i}}.
\end{align*}
We use $P^{\star}_\mathrm{w}$ to denote the fixpoint which is obtained by repeatedly applying $T^\mathrm{w}_P$ starting from the minimal clausal valuation $V=\emptyset$, \ie the least fixpoint of $T^\mathrm{w}_P$ \wrt set inclusion. When $P$ is a positive clausal program we take $\lambda \in \set{0,1}$.
\end{definition}

\begin{example}
 Consider the clausal program $P$ with the clausal rules
 \begin{align*}
  \prulealign{1}{a\lor b \lor c}{}\\
  \prulealign{0.4}{\neg b}{}\\
  \prulealign{0.8}{e}{(a \lor c \lor d)}.
 \end{align*}
 
We can easily verify that, starting from $V = \emptyset$, we obtain 
\begin{align*}
  T^\mathrm{w}_P(V)(a \lor b \lor c) &= 1 \text{ and }\\
  T^\mathrm{w}_P(V)(\neg b) &= 0.4.
\end{align*}
In the next iteration we furthermore find that
\begin{align*}
  T^\mathrm{w}_P(T^\mathrm{w}_P(V))(e) &= 0.4
\end{align*}
since $(\prule{0.8}{e}{(a \lor c \lor d)}) \in P_{0.4}$ and since ${(T^\mathrm{w}_P(V))}^{0.4} \models a \lor c \lor d$. In addition, this is the least fixpoint, \ie we have ${P^{\star}_\mathrm{w} = \set{\plit{(a \lor b \lor c)}{1}, \plit{\neg b}{0.4}, \plit{e}{0.4}}}$. 
\end{example}

Notice that this definition of the immediate consequence operator is a generalization of the immediate consequence operator for possibilistic simple programs (see \pointtodef{nicolas:immediate}). Indeed, for a possibilistic positive clausal program where all clauses contain only a single literal, \ie a possibilistic simple program, we have that $P^{\star} = P^{\star}_\mathrm{w}$. In addition, when all clauses contain only a single literal, we can simplify the immediate consequence operator and simply write $e_i \in V^\lambda$ instead of $V^\lambda \models e_i$.

\noindent We now show that the fixpoint obtained from the immediate consequence operator $T^\mathrm{w}_P$ is indeed the answer set of $P$.

\begin{proposition}
\label{prop:weak:positive:answerset}
Let $P$ be a possibilistic positive clausal program without possibilistic constraint rules. Then $P^{\star}_{\mathrm{w}}$ is a possibilistic answer set of~$P$.
\end{proposition}
\proofat{6--7}

Thus far, we only considered possibilistic positive clausal programs. If we allow for negation-as-failure, we will also need to generalize the notion of a reduct. As usual, in the classical case we want that an expression of the form `$\naf e$' is true when `$e$' cannot be entailed. Furthermore, since we are working in the possibilistic case, we want to take the degrees into account when determining the reduct.

\begin{definition}\label{def:immediate:weak}
Given a possibilistic clausal program $P$ and a valuation $V$, the reduct $P^V$ of $P$ \wrt $V$ is defined as:
\begin{align*}
  P^V = \{\ &((\srule{e}), \min(\lambda_{\mathit{rule}}, \lambda_{\mathit{body}}))~\mid~\min(\lambda_{\mathit{rule}}, \lambda_{\mathit{body}}) > 0\\
  &\land \lambda_{\mathit{body}} = \max\condset{\lambda}{\forall i \in \set{\range{m+1}{n}} \cdot V^{\underline{1-\lambda}} \not\models e_i,\lambda \in [0,1]}\\
  &\land ((\grule{e}), \lambda_{\mathit{rule}}) \in P \}
\end{align*}
\end{definition}
This definition corresponds with the Gelfond-Lifschitz reduct when we consider crisp clausal programs where each clause consists of exactly one literal. Indeed, if we consider clauses with exactly one literal, we could simplify $\forall i \in \set{\range{m+1}{n}} \cdot V^{\underline{1-\lambda}} \not\models e_i$ to $\set{\srange{e}{m+1}{n}} \cap V^{\underline{1-\lambda}} = \emptyset$. This new reduct generalises the Gelfond-Lifschitz reduct in two ways. Firstly, we now have clauses, \ie we now need to verify whether the negative body is not entailed by our guess. 
 Secondly, we need to take the weights attached to the rules, which we interpret as certainties, into account. In particular, the certainty of the reduct of a rule is limited by the certainty of the negative body of the rule and the certainty of the rule itself. In the crisp case these certainty degrees would become trivial.

\begin{proposition}
\label{prop:weak:general:answerset}
A valuation $E$ is a possibilistic answer set of the possibilistic clausal program $P$ without possibilistic constraint rules iff $E$ is a possibilistic answer set of $P^E$.
\end{proposition}
\proofat{7}

\noindent Before we discuss the complexity results, we look at an example to further uncover the intuition of clausal programs.

\begin{example}
Consider the possibilistic clausal program $P$ with the following rules:
\begin{align*}
  \prulealign{0.7}{a \vee b \vee c}{} &
  \prulealign{0.2}{\neg b}{} &
  \prulealign{1}{d}{\naf (a \vee c \vee f)} &
  \prulealign{1}{e}{\naf c}.
\intertext{The reduct $P^V$ with $V = \set{\plit{(a \lor b \lor c)}{0.7}, \plit{(\neg b)}{0.2}, \plit{d}{0.8}, \plit{e}{1}}$ is then:}
  \prulealign{0.7}{a \vee b \vee c}{} &
  \prulealign{0.2}{\neg b}{}&
  \prulealign{0.8}{d}{}&
  \prulealign{1}{e}{}
\end{align*}
since $V^{1-0.8} \models a \vee c$ but $V^{\underline{1-0.8}} \not \models a \vee c$ and $V^{\underline{1-1}} \not \models c$. We then have that ${(P^V)}^\star_\mathrm{w} = \set{\plit{(a \lor b \lor c)}{0.7}, \plit{(\neg b)}{0.2}, \plit{d}{0.8}, \plit{e}{1}}$, hence $V$ is indeed an answer set of $P$.
\end{example}

\section{Complexity Results}\label{sec:complexity}
Before we discuss the complexity results of the weak possibilistic semantics for disjunctive rules~(\pointtosec{disjunction:weak}), we first look at the complexity results of both possibilistic normal programs~(\pointtosec{characterizing:normal}) and the strong possibilistic semantics for disjunctive rules~(\pointtosec{disjunction:strong}). As such, for \pointtoprop{simple:braveNPcomplete}, \ref{prop:simple:cautiouscoNPcomplete}, \ref{prop:simple:braveSigma2complete} and \ref{prop:simple:cautiousPi2complete} we once again consider a valuation $V$ for a possibilistic normal/disjunctive program $P$ as a $\function{V}{\hlit{P}}{[0,1]}$ mapping. We find that for possibilistic normal programs the addition of weights does not affect the complexity compared to classical normal programs.

\begin{proposition}[possibilistic normal program; brave reasoning]
\label{prop:simple:braveNPcomplete}
Let $P$ be a possibilistic normal program. The problem of deciding whether there exists a possibilistic answer set $V$ of $P$ such that $V(l) \geq \lambda$ is $\cNP$-complete.
\end{proposition}
\proofat{8}

\begin{proposition}[possibilistic normal program; cautious reasoning]
\label{prop:simple:cautiouscoNPcomplete}
Let $P$ be a possibilistic normal program. The problem of deciding whether for all possibilistic answer sets $V$ of $P$ we have that $V(l) \geq \lambda$ is $\ccoNP$-complete.
\end{proposition}
\proofat{9}

Similarly, we find for possibilistic disjunctive programs under the strong disjunctive semantics that the addition of weights does not affect the complexity compared to classical disjunctive programs.

\begin{proposition}[possibilistic disjunctive program; brave reasoning]
\label{prop:simple:braveSigma2complete}
Let $P$ be a possibilistic disjunctive program. The problem of deciding whether there is a possibilistic answer set $V$ such that $V(l) \geq \lambda$ is a $\spolsig{2}$-complete problem.
\end{proposition}
\proofat{9-10}

\begin{proposition}[possibilistic disjunctive program; cautious reasoning]
\label{prop:simple:cautiousPi2complete}
Let $P$ be a possibilistic disjunctive program. The problem of deciding whether for all possibilistic answer sets $V$ we have that $V(l) \geq \lambda$ is a $\spolpi{2}$-complete problem.
\end{proposition}
\proofat{10-11}

We now look at the complexity of the weak possibilistic semantics for disjunctive rules for a variety of decision problems and under a variety of restrictions. In~particular, throughout this section we look at the complexity of weak disjunction in the crips case that allows us to compare these results against the complexity of the related decision problems in classical ASP and other epistemic extensions of ASP, \eg~\cite{truszczynski:revisiting,vlaeminck:ordered}. As we will see, for certain classes of clausal programs, decision problems exist where weak disjunction is computationally less complex than disjunctive programs while remaining more complex than normal programs.

 An overview of the complexity results available in the literature for disjunctive programs as well as the new results for weak disjunction (in the crisp case) which we discuss in the remainder of this section can be found in \pointtotbl{results-extra}.

\begin{table}[ht]
\caption{Completeness results for the main reasoning tasks with references}
\label{tbl:results-extra}
\centering
    \begin{tabular}{rccc}\hline
        no NAF, no $\neg$     & existence & brave reasoning     & cautious reasoning \\\hline
        strong disjunction                                & $\cNP$~\tlabel{1}            & $\spolsig{2}$~\tlabel{1}       & $\ccoNP$~\tlabel{1}           \\ 
        weak disjunction                                  & $\cP$~\tlabel{6}                & $\cP$~\tlabel{6}               & $\cP$~\tlabel{6}              \\ \hline\hline
        no NAF, $\neg$   & existence & brave reasoning     & cautious reasoning \\ \hline
        strong disjunction                                & $\cNP$~\tlabel{1}               & $\spolsig{2}$~\tlabel{1}       & $\ccoNP$~\tlabel{1}           \\ 
        weak disjunction                                  & $\cNP$~\tlabel{4}              & $\complexity{BH}_2$~\tlabel{3} & $\ccoNP$~\tlabel{5}           \\ \hline\hline
        NAF, $\neg$ & existence & brave reasoning     & cautious reasoning \\ \hline
        strong disjunction                                & $\spolsig{2}$~\tlabel{2}        & $\spolsig{2}$~\tlabel{2}       & $\spolpi{2}$~\tlabel{2}       \\
        weak disjunction                                  & $\spolsig{2}$~\tlabel{8}        & $\spolsig{2}$~\tlabel{7}       & $\spolpi{2}$~\tlabel{9}       \\ 
        \hline
        \multicolumn{4}{c}{\tiny ``no NAF'' (\resp ``no $\neg$'') indicates results for programs without negation-as-failure (\resp classical negation)}\\
        \multicolumn{4}{c}{~}
        \end{tabular}
        \begin{tabular}{llll}
        \multicolumn{2}{l}{\tref{1}{\cite{eiter:complexity}}}&\multicolumn{2}{l}{\tref{6}{\pointtoprop{fixpoint:polynomial}}}\\
        \multicolumn{2}{l}{\tref{2}{\cite{baral:knowledge}}}&\multicolumn{2}{l}{\tref{7}{\pointtoprop{sigma2hardness} and \ref{prop:sigma2membership}}}\\
        \multicolumn{2}{l}{\tref{3}{\pointtoprop{bh2hardness} and \ref{prop:bh2membership}}}&\multicolumn{2}{l}{\tref{8}{\pointtocor{general:existence}}}\\
        \multicolumn{2}{l}{\tref{4}{\pointtocor{weak:general:existence}}}&\multicolumn{2}{l}{\tref{9}{\pointtocor{general:cautious}}}\\
        \multicolumn{2}{l}{\tref{5}{\pointtocor{weak:general:cautious}}}
        \end{tabular}
\end{table}

\begin{proposition}[weak disjunction, positive clausal program; brave reasoning]
\label{prop:bh2hardness}
Let $P$ be a positive clausal program. The problem of deciding whether a clause `$e$' is entailed by a consistent answer set $E$ of $P$ is $\complexity{BH}_2$-hard.
\end{proposition}
\proofat{11-12}

\begin{proposition}[weak disjunction, positive clausal program; brave reasoning]
\label{prop:bh2membership}
  Let $P$ be a positive clausal program. The problem of deciding whether a clause `$e$' is entailed by a consistent answer set $M$ of $P$ is in $\complexity{BH}_2$.
\end{proposition}
\proofat{12-13}

\begin{corollary}
\label{cor:complete}
  Let $P$ be a positive clausal program. The problem of deciding whether a clause `$e$' is entailed by a consistent answer set $E$ of $P$ is $\complexity{BH}_2$-complete.
\end{corollary}

\begin{corollary}[weak disjunction, positive clausal program; answer set existence]
\label{cor:weak:general:existence}
  Determining whether a positive clausal program $P$ has a consistent answer set is an $\cNP$-complete problem. 
\end{corollary}
\proofat{14}

\begin{corollary}[weak disjunction, positive clausal program; cautious reasoning]
\label{cor:weak:general:cautious}
  Cautious reasoning, \ie determining whether a clause `$e$' is entailed by every answer set $M$ of a positive clausal program $P$ is $\ccoNP$-complete.
\end{corollary}
\proofat{14}

Surprisingly, the expressivity of positive clausal programs under the weak interpretation of disjunction is directly tied to the ability to use classical negation in clauses. If we limit ourselves to positive clausal programs without classical negation we find that the expressiveness is restricted to \cP.

In order to see this, let us take a closer look at the immediate consequence operator for clausal programs as defined in \pointtodef{weak:immediate}. When there are no occurrences of classical negation we can simplify this immediate consequence operator to
\begin{align*}
{T}^\mathrm{w}_P(E) = \condset{e_0}{\srule{e} \in P \loand \forall i \in \set{\range{1}{m}} \cdot \exists e \in E \cdot e \subseteq e_i}
\end{align*}
where $e \subseteq e_i$ is defined as the subset relation where we interpret $e$ and $e_i$ as sets of literals, \ie $e = (\srangec{l}{1}{n}{\lor})$ is interpreted as $\set{\srange{l}{1}{n}}$.

\begin{proposition}
\label{prop:fixpoint:polynomial}
  Let $P$ be a positive clausal program without classical negation. We can find the unique answer set of $P$ in polynomial time.
\end{proposition}
\proofat{14}

\noindent We now examine the complexity of general clausal programs. We will do this by showing that the problem of determining the satisfiability of a QBF of the form $\phi = \exists X_1 \forall X_2 \cdot p(X_1, X_2)$ with $p(X_1, X_2)$ in DNF can be reduced to the problem of determining whether a clause `$e$' is entailed by a consistent answer set $M$ of the clausal program $P$. We start with the definition of our reduction.

\begin{definition}\label{def:simulationQBF}
  Let $\phi = \exists X_1 \forall X_2 \cdot p(X_1,X_2)$ be a QBF with $p(X_1,X_2) = \srangec{\theta}{1}{n}{\lor}$ a formula in disjunctive normal form with $X_i$ sets of variables.
  We define the clausal program $P_\phi$ corresponding to $\phi$ as
  \begin{align}
    P_{\phi} = &\condset{\arule{x}{\naf \neg x}}{x \in X_1} \cup \condset{\arule{\neg x}{\naf x}}{x \in X_1}\label{eq:rules:generate}\\
           & \cup \condset{\arule{\neg \theta_t \lor sat}{}}{1 \leq t \leq n}\label{eq:rules:entailment}\\
           & \cup \set{\arule{}{\naf sat}}\label{eq:rules:constraint}
  \end{align}
  with $\neg \theta_t$ the clausal representation of the negation of the formula $\theta_t$, \eg when $\theta_t = \rangec{x_1 \land \neg x_2}{\neg x_k}{\land}$ then $\neg \theta_t = \rangec{\neg x_1 \lor x_2}{x_k}{\lor}$.
\end{definition}

\begin{example}
  Given the QBF $\phi = \exists p_1, p_2 \forall q_1, q_2 \cdot (p_1 \land q_1) \lor (p_2 \land q_2) \lor (\neg q_1 \land \neg q_2)$ the clausal program $P_\phi$ is 
  \begin{align*}
    \arule{p_1&}{\naf \neg p_1}\\
    \arule{\neg p_1&}{\naf p_1}\\
    \arule{p_2&}{\naf \neg p_2}\\
    \arule{\neg p_2&}{\naf p_2}\\
    \arule{\neg p_1 \vee \neg q_1 \vee \mathit{sat}&}{}\\
    \arule{\neg p_2 \vee \neg q_2 \vee \mathit{sat}&}{}\\
    \arule{q_1 \vee q_2 \vee \mathit{sat}&}{}\\
    \arule{&}{\naf \mathit{sat}}.
  \end{align*}
  Notice how $M = \set{p_1, p_2, \neg p_1 \vee \neg q_1 \vee \mathit{sat}, \neg p_2 \vee \neg q_2 \vee \mathit{sat}, q_1 \vee q_2 \vee \mathit{sat}}$ is an answer set of $P_\phi$ and that $M \models \mathit{sat}$. Accordingly we find that the QBF is satisfied.
  
 If we take the QBF $\phi' = \exists p_1, p_2 \forall q_1, q_2 \cdot (p_1 \land q_1) \lor (p_2 \land q_2)$ then the clausal program $P_{\phi'}$ corresponding to $\phi'$ is the program $P_\phi$ in which the penultimate rule has been removed. Notice how $P_{\phi'}$ has no answer sets, because we are not able to entail `$\mathit{sat}$' from any of the answer sets of $P_{\phi'}$. Indeed, the QBF $\phi'$ is not satisfiable.
\end{example}

\begin{proposition}[weak disjunction; brave reasoning]
\label{prop:sigma2hardness}
Let $P$ be a clausal program. The problem of deciding whether a clause `$e$' is entailed by a consistent answer set $M$ of $P$ is $\spolsig{2}$-hard.
\end{proposition}
\proofat{14-15}

\begin{proposition}[weak disjunction; brave reasoning]
\label{prop:sigma2membership}
  Let $P$ be a clausal program. The problem of deciding whether a clause `$e$' is entailed by a consistent answer set $M$ of $P$ is in $\spolsig{2}$.
\end{proposition}
\proofat{15}

\begin{corollary}
\label{cor:general:complete}
  Let $P$ be a clausal program. The problem of deciding whether a clause `$e$' is entailed by a consistent answer set $E$ of $P$ is $\spolsig{2}$-complete.
\end{corollary}

\begin{corollary}[weak disjunction; answer set existence]
\label{cor:general:existence}
Determining whether a clausal program $P$ has a consistent answer set is an $\spolsig{2}$-complete problem.
\end{corollary}
\proofat{15}

\begin{corollary}[weak disjunction; cautious reasoning]
\label{cor:general:cautious}
  Cautious reasoning, \ie determining whether a clause `$e$' is entailed by every answer set $M$ of a clausal program $P$, is $\spolpi{2}$-complete.
\end{corollary}

\begin{proof}
This problem is complementary to brave reasoning, \ie we verify that there does not exist an answer set $M'$ of $P$ such that `$\neg e$' is entailed by $M'$.
\end{proof}

\section{Related Work}\label{sec:related}
The work presented in this paper touches on various topics that have been the subject of previous research. In this section we structure our discussion of related existing work along 3 main lines. Previous work on the semantics of disjunctive programs is discussed in \pointtosec{related:disjunctive}. In \pointtosec{related:epistemic} we look at how ASP and possibility theory have been used in the literature for epistemic reasoning. Finally, in \pointtosec{related:rulesas}, we look at prior work on characterizing rules with possibility theory and fuzzy logic.

\subsection{Semantics of Disjunctive Programs}\label{sec:related:disjunctive}
\noindent Many characterizations of stable models have been proposed in the literature. We refer the reader to~\cite{lifschitz:thirteen} for a concise overview of thirteen such definitions. 
One of the earliest characterizations of stable models was in terms of autoepistemic logic~\cite{moore:semantical}. Formulas in autoepistemic logic are constructed using atoms and propositional connectives, as well as the modal operator $\complexity{L}$, which intuitively stands for ``\emph{it is believed}''. The characterization of stable models proposed in~\cite{gelfond:classical} based on autoepistemic logic is to look at `$\naf a$' as the expression `$\neg\complexity{L}a$', a choice which clearly stands out for its simplicity and intuitively. For example, to explain the semantics of the rule $\grule{a}$ one would consider the formula $\srangec{a}{1}{m}{\land} \land\psrangec{\neg \complexity{L}}{a}{m+1}{n}{\land}\rightarrow a_0$. Yet this characterization does have some problems. Indeed, it was soon afterwards realized that this correspondence does not hold for programs with classical negation or disjunction in the head. A more involved characterization based on autoepistemic logic that does work for classical negation and disjunction has been proposed in~\cite{lifschitz:extended}. The idea is to look at literals `$l$' that are not preceded by negation as failure as the formula ($l \land \complexity{L}l$), while one still looks at a literal of the form `$\naf l$' as the formula $\neg \complexity{L}l$. In our approach, an expression of the form `$\naf l$' is essentially identified with $\Pi(\neg l)$, which clearly resembles the first characterization in terms of autoepistemic logic. By staying closer to the Gelfond-Lifschitz reduct, our approach is more elegant in that we do not require a special translation of literals in order to be able to deal with classical negation and disjunction.

Several authors have already proposed alternatives and extensions to the semantics of disjunctive programs. Ordered disjunction~\cite{brewka:logic} falls in the latter category and allows to use the head of the rule to formulate alternative solutions in their preferred order. For example, a rule such as $\arule{\srangec{l}{1}{k}{\times}}{}$ represents the knowledge that $l_1$ is preferred over $l_2$ which is preferred over $l_3$ \ldots , but that at the very least we want $l_k$ to be true. As such it allows for an easy way to express context dependent preferences. The semantics of ordered disjunction allow certain non-minimal models to be answer sets, hence, unlike the work in this paper, it does not adhere to the standard semantics of disjunctive rules in ASP. 

Annotated disjunctions are another example of a framework that changes the semantics of disjunctive programs~\cite{vennekens:logic}. It is based on the idea that every disjunct in the head of a rule is annotated with a probability. Interestingly, both ordered and annotated disjunction rely on split programs, as found in the possible model semantics~\cite{sakama:alternative}. These semantics provide an alternative to the minimal model semantics. The idea is to split a disjunctive program into a number of normal programs, one for each possible choice of disjuncts in the head, of which the minimal Herbrand models are then the possible models of the disjunctive programs. Intuitively this means that a possible model represents a set of atoms for which a possible justification is present in the program. In line with our conclusions for weak disjunction, using the possible model semantics also leads to a lower computational complexity.

Not all existing extensions of disjunction allow non-minimal models. For example, in \cite{buccafurri:disjunctive} an extension of disjunctive logic programs is presented which adds the idea of inheritance. Conflicts between rules are then resolved in favor of more specific rules. Such an approach allows for an intuitive way to deal with default reasoning and exceptions. In particular, the semantics allow for rules to be marked as being defeasible and allows to specify an order or inheritance tree among (sets of) rules. Interestingly, the complexity of the resulting system is not affected and coincides with the complexity of ordinary disjunctive programs.

\subsection{Epistemic Reasoning with ASP and Possibility Theory}\label{sec:related:epistemic}
In~\cite{gelfond:strong} it was argued that classical ASP, while later proven to have strong epistemic foundations~\cite{loyer:epistemic}, is not well-suited for epistemic reasoning. Specifically, ASP lacks mechanisms for introspection and can thus not be used to \eg reason based on cautiously deducible information. At the same time, however, it was shown that extensions of ASP could be devised that do allow for a natural form of epistemic reasoning. The language $ASP^{K}$ proposed in~\cite{gelfond:strong} allows for modal atoms, \eg $\mathsf{K}a$, where $\mathsf{K}$ is a modal operator that can intuitively be read as ``\emph{it is known that [$a$ is true]}''. These new modal atoms can in turn be used in the body of rules. The semantics of $ASP^{K}$ were originally based on a three-valued interpretation (to allow for the additional truth value `$uncertain$'), but later, in~\cite{truszczynski:revisiting}, it was shown that this is not essential and that a more classical two-valued possible world structure can also be considered. In addition, further extensions are discussed that allow for epistemic reasoning over arbitrary theories, where it is shown that $ASP^{K}$ can be encoded within these extensions. The complexity is studied for these extensions and is shown to be brought up one level \wrt ASP, \eg to \spolsig{3} for disjunctive epistemic programs.

Alternatively, existing extensions of ASP can be used to implement some epistemic reasoning tasks, such as reasoning based on brave/cautious conclusions. This idea is proposed in~\cite{faber:manifold} to overcome the need for an intermediary step to compute the desired consequences of the ASP program $P_1$, before being fed into $P_2$. Rather, they propose a translation to manifold answer set programs, which exploit the concept of weak constraints~\cite{buccafurri:enhancing} to allow for such programs to access all desired consequences of $P_1$ within a single answer set. As such, for problems that can be cast into this particular form, only a single ASP program needs to be evaluated and the intermediary step is made obsolete. 

As we mentioned in \pointtosec{related:disjunctive}, the semantics of ASP can also be expressed in terms of autoepistemic logic~\cite{marek:autoepistemic}. These semantics have the benefit of making the modal operator explicit, allowing for an extensions of ASP that incorporates such explicit modalities to better express exactly which form of knowledge is required. However, since autoepistemic logic treats negation-as-failure as a modality, it is quite hard to extend to the uncertain case. Furthermore, as already discussed, it as shown in~\cite{lifschitz:extended} that this characterization does not allow us to treat classical negation or disjunctive rules in a natural way, which weakens its position as a candidate for generalizing ASP from normal programs to \eg disjunctive programs. 

Possibility theory, which can \eg be used for belief revision, has a strong epistemic notion and shares a lot of commonalities with epistemic entrenchments~\cite{dubois:epistemic}. Furthermore, in~\cite{dubois:stable} a generalization of possibilistic logic is studied, which corresponds to a weighted version of a fragment of the modal logic KD. In this logic, epistemic states are represented as possibility distributions, and logical formulas are used to express constraints on possible epistemic states. In this paper, we similarly interpret rules in ASP as constraints on possibility distributions, which furthermore allows us to unearth the semantics of weak disjunction.

\subsection{Characterization of Rules using Possibility Theory and Fuzzy Logic}\label{sec:related:rulesas}
\noindent A large amount of research has focused on how possibility distributions can be used to assign a meaning to rules. For example, possibility theory has been used to model default rules~\cite{benferhat:representing,benferhat:nonmonotonic}. Specifically, a default rule ``if $a$ then $b$'' is interpreted as $\Pi(a\wedge b) > \Pi(a\wedge \neg b)$, which captures the intuition that when $a$ is known to hold, $b$ is more plausible than $\neg b$, if all that is known is that $a$ holds. In this approach entailment is defined by looking at the least specific possibility distributions which is similar in spirit to our approach for characterizing ASP rules (although the notion of least specific possibility distribution is defined, in this context, \wrt the plausibility ordering on interpretations induced by the possibility degrees). 

The work on possibilistic logic~\cite{dubois:possibilistic} forms the basis of possibilistic logic programming~\cite{dubois:towards}. The idea of possibilistic logic programming is to start from a necessity-valued knowledge base, which is a finite set of pairs $(\phi~\alpha)$, called necessity-valued formulas, with $\phi$ a closed first-order formula and $\alpha \in [0, 1]$. Semantically, a necessity-valued formula expresses a constraint of the form $N(\phi) \geq \alpha$ on the set of possibility distributions. A possibilistic logic program is then a set of necessity-valued implications. As rules are essentially modelled using material implication, however, the stable model semantics cannot straightforwardly be characterized using possibilistic logic programming. For example, the knowledge base $\set{(a \rightarrow b~~1), (\neg b~~1)}$, which represents the program $\set{\arule{b}{a}, \arule{\neg b}{}}$, induces that $N(\neg a) = 1$. Indeed, the semantics of this knowledge base indicate that $\Pi(a \land \neg b) = 0$ and $\Pi(b) = 0$, \ie we find that $\Pi(a) = 0$. In other words: a direct encoding using possibilistic logic programming allows for contraposition, which is not in accordance with the stable model semantics.

Rules in logic can also be interpreted as statements of conditional probability~\cite{jaynes:probability}. In the possibilistic setting this notion has been adapted to the notion of conditional necessity measures. Rules can be then also be modelled in terms of conditional necessity measures~\cite{benferhat:nonmonotonic,dubois:synthetic,benferhat:transformation}. The conditional possibility measure $\condPi{\psi}{\phi}$ is defined as the greatest solution to the equation $\Pi(\phi \wedge \psi) = \min(\condPi{\psi}{\phi}, \Pi(\phi))$ in accordance with the principle of least specificity. It can be derived mathematically that this gives us $\condPi{\psi}{\phi} = 1$ if $\Pi(\psi \wedge \phi) = \Pi(\phi)$ and $\condPi{\psi}{\phi} = \Pi(\psi \wedge \phi)$ otherwise whenever $\Pi(\phi) > 0$. When $\Pi(\phi) = 0$, then by convention $\condPi{\psi}{\phi} = 1$ for every  $\psi \neq \bot$ and $\condPi{\bot}{\phi} = 0$, otherwise.  The conditional necessity measure is defined as $\condN{\psi}{\phi} = 1 - \condPi{\neg \psi}{\phi}$. However, there does not seem to be a straightforward way to capture the stable model semantics using conditional necessity measures, especially when classical negation is allowed. Indeed, if we represent the semantics of the program $\set{\arule{b}{a}, \arule{\neg b}{}}$ as the constraints $\condN{b}{a} \geq 1$ and $\condN{\neg a}{\top} \geq 1$. Using the definition of the conditional necessity measure, the first constraint is equivalent to $1-\condPi{\neg b}{a} \geq 1$, \ie $\condPi{\neg b}{a} = 0$. The second constraint simplifies to $\Pi(a) = 0$, which, using the convention stated above gives rise to $\condPi{\neg b}{a} = 1$, clearly a contradictory result to the earlier conclusion that $\condPi{\neg b}{a} = 0$.

The work in~\cite{nicolas:possibilistic} was one of the first papers to explore the idea of combining possibility theory with ASP. Rather than defining the semantics of ASP in terms of constraints on possibility distributions as we did in this paper, the goal was to allow one to reason with possibilities in ASP programs. In this way one can associate certainties with the information encoded in an ASP program. The~approach from~\cite{nicolas:possibilistic} upholds the 1-on-1 relationship between the classical answer sets of a normal program and the possibilistic answer sets, which brings with it some advantages. One of those advantages is that it allows us to deal with possibilistic nested programs~\cite{nieves:nested}. The work from Nicolas et al. was later extended to also cover the case of disjunctive ASP in~\cite{nieves:semantics}. The~latter approach allows us to \eg capture qualitative information by considering partially ordered sets, which would not be straightforward to accomplish in our work. 
 However, the~approaches from~\cite{nicolas:possibilistic} and~\cite{nieves:semantics} work by taking a possibilistic ASP program and reducing it --~by~ignoring the certainty values~-- to a possibilistic ASP program without negation-as-failure. As~such, both approaches loose the certainty encoded through negation-as-failure, since the certainty values are not taken into account. 

Possibility theory has also been used to define various semantics of fuzzy if-then rules~\cite{zadeh:fuzzy2}. Rather than working with literals, fuzzy if-then rules consider fuzzy predicates which each have their own universe of discourse. To draw conclusions from a set of fuzzy if-then rules, mechanisms are needed that can produce an (intuitively acceptable) conclusion from a set of such rules.

Finally, a formal connection also exists between the approach from \pointtosec{characterization} and the work on residuated logic programs~\cite{damasio:monotonic} under the G\"odel semantics. Both approaches are different in spirit, however, in the same way that possibilistic logic (which deals with uncertainty or priority) is different from G\"odel logic (which deals with graded truth). The formal connection is due to the fact that necessity measures are min-decomposable and disappears as soon as classical negation or disjunction is considered.

\section{Conclusions}\label{sec:conclusions}
In this paper we defined semantics for Possibilistic ASP (PASP), a framework that combines possibility theory and ASP to allow for reasoning under (qualitative) uncertainty. These semantics are based on the interpretation of possibilistic rules as constraints on possibility distributions. We showed how our semantics for PASP differ from existing semantics for PASP. Specifically, they adhere to a different intuition for negation-as-failure. As such, they can be used to arrive at acceptable results for problems where the possibilistic answer sets according to the existing semantics for PASP do not necessarily agree with our intuition of the problem. In~addition, we showed how our semantics for PASP allowed for a new characterization of ASP. When looking at ASP as a special case of PASP, we naturally recover the intuition of a rule that the head is certain whenever we are certain that the body holds. The resulting characterization stays close to the intuition of the stable model semantics, yet also shares the explicit reference to modalities with autoepistemic logic. We showed that this characterization not only naturally characterizes normal programs, \ie programs with negation-a-failure, but can also naturally characterize disjunctive programs and programs with classical negation. 

Due to our explicit reference to modalities in the semantics, we are furthermore able to characterize an alternative semantics for disjunction in the head of a rule that has a more epistemic flavour than the standard treatment of disjunction in ASP, \ie given a rule of the form $(\arule{a \vee b}{})$ we do not obtain two answer sets, but rather we have `$a \vee b$' as-is in the answer set. While such a characterization might seem weak, we showed that the interplay with literals significantly affects the expressiveness. Indeed, we found that the problem of brave reasoning/cautious reasoning under these weak semantics for disjunction for a program without negation-as-failure, but with classical negation, is $\complexity{BH}_2$-complete and $\ccoNP$-complete, respectively. This highlights that weak disjunction is not merely syntactic sugar, \ie it cannot simply be simulated in normal ASP without causing an exponential blow-up. For strong disjunction, on the other hand, we have obtained that brave and cautious reasoning without negation-as-failure are $\spolsig{2}$-complete and $\ccoNP$-complete, respectively. 
As such, the weak semantics for disjunction detailed in this paper allow us to work with disjunction in a less complex way that still remains non-trivial. If, however, we restrict ourselves to atoms, then brave reasoning under the weak semantics for disjunction is $\cP$-complete.

\end{document}